\documentclass[10pt,journal,cspaper,compsoc]{IEEEtran}
\usepackage{amssymb, amsmath, latexsym, amsfonts, mathrsfs, amsbsy}
\usepackage[dvips]{graphicx}
\graphicspath{./Images/}
\usepackage[cp1252, latin1]{inputenc}
\usepackage{epsfig, color}
\usepackage{listings}
\usepackage{setspace}
\usepackage{multicol}
\usepackage{makeidx, glossaries}
\usepackage{url}
\usepackage{flushend}
\usepackage{multicol}
\usepackage{supertabular}
\usepackage[caption=false, font=footnotesize]{subfig}
\usepackage{mathbbol}
\usepackage{stmaryrd}
\usepackage{enumerate}
\usepackage{epigraph}
\usepackage{multirow}
\usepackage{footnote}
\usepackage{bbm}

\usepackage[T1]{fontenc}
\usepackage{calligra}

\newtheorem{theorem}{Theorem}
\newtheorem{lemma}{Lemma}
\newtheorem{proof}{Proof}

\newtheorem{proposition}{Proposition}
\newtheorem{remark}{Remark}

\usepackage{color}
\usepackage{longtable}
\usepackage[algoruled,noend,noline]{algorithm2e}
\usepackage{lipsum}

\ifCLASSOPTIONcompsoc
\else
\fi
\ifCLASSINFOpdf
\else
\fi
\begin{document}
\title{Distributed $k$-means algorithm}

\author{Gabriele~Oliva,~Roberto~Setola,~and~Christoforos~N.~Hadjicostis% <-this % stops a space
\IEEEcompsocitemizethanks{\IEEEcompsocthanksitem 
G. Oliva and R. Setola are with the Complex Systems \& Security Laboratory, University Campus Bio-Medico, via A. del Portillo 21, 00128, Rome, Italy.\protect\\
E-mail: g.oliva@unicampus.it, r.setola@unicampus.it}% <-this % stops a space
\IEEEcompsocitemizethanks{\IEEEcompsocthanksitem Christoforos~N.~Hadjicostis is with the Department of Electrical and Computer Engineering, University of Cyprus, 75 Kallipoleos Avenue, P.O. Box 20537, 1678 Nicosia
Cyprus.\protect\\
E-mail: chadjic@ucy.ac.cy}% <-this % stops a space
\thanks{}}


\markboth{submitted to IEEE Transactions on Mobile Computing}%
{Shell \MakeLowercase{\textit{et al.}}: Bare Demo of IEEEtran.cls for Computer Society Journals}

%\IEEEpubid{\makebox[\columnwidth]{\hfill 0000--0000/00/\$00.00~\copyright~2007 IEEE}%

% PRIOR to the title within the \IEEEcompsoctitleabstractindextext IEEEtran
\IEEEcompsoctitleabstractindextext{%
\begin{abstract}
%\boldmath
In this paper we provide a fully distributed implementation of the $k$-means clustering algorithm, intended for wireless sensor networks where each agent is endowed with a possibly high-dimensional observation (e.g., position, humidity, temperature, etc.).
The proposed algorithm, by means of one-hop communication, partitions the agents into measure-dependent groups that have small in-group and large out-group ``distances".
%Depending on the nature of the measurement performed by each agent, it is possible to obtain a position clustering (i.e., grouping the agents according to their position), a measure clustering (i.e., grouping based on the sensors'measurement such as temperature, humidity, etc.) or a hybrid clustering (i.e., considering both the position and the measurement of each agent).
Since the partitions may not have a relation with the topology of the network--members of the same clusters may not be spatially close--the algorithm is provided with a mechanism to compute the clusters'centroids even when the clusters are disconnected in several sub-clusters.The results of the proposed distributed algorithm coincide, in terms of minimization of the objective function, with the centralized $k$-means algorithm.
Some numerical examples illustrate the capabilities of the proposed solution. 
\end{abstract}
% not want either math or citations to appear in the abstract.

% Note that keywords are not normally used for peer review papers.
}

\maketitle

% is not selected <OR> if conference mode is selected - because compsoc
% conference papers position the abstract like regular (non-compsoc)
\IEEEdisplaynotcompsoctitleabstractindextext

% \fi
\IEEEpeerreviewmaketitle

\section{Introduction}

Wireless sensor networks are increasingly asserting their presence in everyday life, as a powerful tool to measure and manage spatially dispersed information in several applications, such as monitoring, robot coordination, indoor localization, rescue, etc.
Dealing with huge sets of multidimensional data is, in many cases, a non-trivial issue, and there is the need to provide agile and efficient methodologies to grasp the fundamental aspects of the ongoing situation.

In the literature, several distributed techniques have been provided in order to let a network of agents reach distributed agreement (or {\em consensus}) based on iterative strategies that rely at each iteration on local observations (for instance,  \cite{Olfati1,Olfati2,cortes2008distributed}). 

The resulting agreement, however, is on a single value for all the agents in the network, while in many circumstances it might be of interest to obtain multiple values, as well as to identify sets of observations that are homogeneous according to some criteria.
In the literature, a powerful, yet centralized, solution to the latter problem is provided by means of the so called {\it clustering algorithms}. One of the most widely adopted clustering algorithms is the $k$-means algorithm, as well as its extensions such as fuzzy $c$-means \cite{Dunn:1973} or Mixture of Gaussians \cite{Dempster:1977} algorithms. The above approaches are examples of unsupervised learning algorithms, where a set of observations has to be partitioned in a finite number  of groups (eventually considering soft partitions in the case of fuzzy c-means) with small in-group and large out-group distances, or has to be explained by a combination of probability distributions (Mixture of Gaussians).

The $k$-means algorithm in particular, is an iterative procedure where a set of $k$ centroids are obtained 
by alternating assignment phases (i.e., each observation is assigned to the nearest current centroid) and refinement phases where the average of the observations in the same partition becomes the new centroid.

In this paper, we provide a fully distributed implementation of the $k$-means algorithm for a network of agents, each holding a possibly high-dimensional observation or piece of information.  
%The algorithm is able to partition the agents according to their actual position in the space (position clustering), to the values measured by the agents (measure clustering) or to a combination of both position and measures (hybrid clustering).

\subsection{Paper Outline}
The outline of the paper is as follows:
after two subsections on notation and definitions that conclude this  introduction, Section \ref{kmeans} reviews the $k$-means clustering algorithm, while Section \ref{relwork} provides a discussion on related work and on the contribution of the paper; Section \ref{consensus} outlines distributed consensus algorithms; Section \ref{section4:distrib} introduces the distributed $k$-means algorithm, while section \ref{correctness_computational_complexity} provides a discussion on the correctness and complexity of the proposed algorithm; simulation results and conclusions are provided in Sections \ref{results} and \ref{conclusions}, respectively. 

\subsection{Notation}
\begin{supertabular}[l]{p{0.5in}p{2.8in}}
%generic
$A'$&transpose of a matrix/vector $A$;\\
$\#(A)$&cardinality of a set $A$;\\
$A\otimes B$&Kronecker product of matrix/vector $A$ and $B$;\\
$I_a$&$a\times a$ identity matrix;\\
$1_a$& $a\times 1$ vector whose elements are all equal to $1$;\\
$0_a$& $a\times 1$ vector whose elements are all equal to $0$;\\
$e_i$& $i$-th vector of the canonical basis;\\
%graphs
$G$&graph;\\
$\mathcal{V}$& set of vertices of the graph $G$;\\
$\mathcal{E}$& set of edges of the graph $G$;\\
$v_i$& $i$-th vertex of the graph $G$;\\
$(v_i,v_j)$& directed edge connecting two vertices $v_i$ (source)  and $v_j$ (destination) of graph $G$;\\
$\mathcal{N}_i$& neighborhood of agent $i$;\\
%steps
$T$&identifier of a step in the (distributed) $k$-means algorithm;\\
$t$&identifier of an iteration in a distributed consensus algorithm;\\
%section 2
$n$&	number of agents/observations;\\
$d$& size of the observations;\\
$x_i$&$d\times 1$ vector of the $i$-th observation;\\
$k$&number of clusters for $k$-means algorithm;\\
$S_i(T)$& $i$-th cluster at step $T$;\\
$r_{ij}(T)$& boolean decision variable representing the assignment of observation $x_i$ to cluster $S_j(T)$ at step $T$;\\
$c_j(T)$& $d\times 1$ vector of the centroid of cluster $S_j(T)$ at step $T$;\\
$D(T)$& objective functional of the $k$-means algorithm at step $T$;\\
$M$&maximum number of iterations of the $k$-means algorithm;\\
$\Delta_{\max}$&positive parameter used to terminate the $k$-means algorithm;\\
%section 3
$z_i(t)$&generic, possibly vectorial, state assumed by the $i$-th agent of a consensus algorithm at iteration $t$;\\
$z_{i0}$&generic, possibly vectorial, initial condition assumed by the $i$-th agent of a consensus algorithm at iteration $t$;\\
$u_i(t)$& input for the $i$-th agent in consensus algorithms at iteration t;\\
%$u_i(\mathcal{N}_i\cup \{i\},t)$& neighborhood-dependent input for the $i$-th agent in consensus algorithms at iteration t;\\
$\chi(\cdot)$& function to be computed by $\chi$-consensus;\\
$t_{\max}$&maximum number of iterations for consensus algorithms;\\
$\overline z_i$&final value of the state of the $i$-th agent in finite-time consensus algorithms;\\
$w_{ij}$&coefficients that compose the input $u_i(\mathcal{N}_i\cup \{i\},t)$ of average-consensus algorithm;\\
$W$& $n\times n$ matrix containing all the terms $w_{ij}$;\\
$z(t)$&stack vector containing the state $z_i(t)$ of each agent at iteration $t$ of a consensus algorithm;\\
$\delta_i$&number of steps necessary to calculate the average-consensus value in finite time for the $i$-th agent;\\
$\gamma_{ij}$& $j$-th coefficient used by the $i$-th agent to calculate the average-consensus value in finite time;\\
$\Gamma_i$&$\delta_i\times 1$ vector containing the terms $\gamma_{ij}$;\\
$q_i(g)$&minimal polynomial of agent $i$;\\
$\beta_{ij}$&coefficients used to calculate $\gamma_{ij}$;\\
$\alpha_{ij}$& coefficients used to calculate the minimal polynomial $q_i(g)$;\\
$\alpha_i$&$(\delta_i+1)\times 1$ vector containing the terms $\alpha_{ij}$;\\
$\Theta_i$&Observability matrix used to calculate the minimal polynomial $q_i(g)$;\\
$z_{*,j}(t)$&$n\times 1$ vector containing the state of the agents when the $j$-th preparatory average-consensus is executed in the finite-time consensus;\\
$Z_{i,\delta_i}$&$n\times (\delta_i+2)$ matrix containing up to a step $\delta_i$ the state of the agent $i$ in all the preparatory average-consensus runs executed in the finite-time consensus;\\
%section4
$\tilde n$&upper bound of the number of agents $n$;\\
$i^*$&index of the leader in the distributed $k$-means algorithm;\\
$c_{ij}(T)$&$d\times 1$ vector containing the value of the $j$-th centroid as computed by the $i$-th agent at step $T$ of the distributed $k$-means algorithm;\\
$c_i(T)$&$kd\times 1$ stack vector containing the terms $c_{ij}(T)$ for all $j=1,\ldots,k$;\\
$\overline j_i$&index of the centroid selected by agent $i$ at step $T$ of the distributed $k$-means algorithm;\\
$c_{i\overline j_i}(T)$&$d\times 1$ vector containing the value of the centroid chosen by agent $i$ at step $T$ of the distributed $k$-means algorithm;\\
$\mu_i(T)$&$k\times 1$ vector used to represent the choice of a centroid for agent $i$ at step $T$ of the distributed $k$-means algorithm;\\
$\hat \mu_i(T)$&$k\times 1$ vector used to represent the choice of a centroid for agent $i$ at step $T$ of the distributed $k$-means algorithm;\\
$c_i^0(T)$&$kd\times 1$ vector used as initial condition of a max-consensus algorithm in order to let each agent calculate the value of the current centroids at step $T$ of the distributed $k$-means algorithm;\\
$\mathcal{N}_i^c(T)$&subset of $\mathcal{N}_i$ containing agents with the same choice of the nearest center at step $T$ of the distributed $k$-means algorithm;\\
$G^c(T)$& subgraph of $G$ induced by $\mathcal{N}_i^c(T)$ for all $i=1, \ldots, n$;\\
$SCC_{jh}(T)$&$d\times 1$ vector representing the centroid of the $h$-th sub-cluster of the $j$-th cluster at step $T$ of the distributed $k$-means algorithm;\\
$i^*_{jh}(T)$& leader of the $h$-th sub-cluster of the $j$-th cluster at step $T$ of the distributed $k$-means algorithm;\\
$SCS_{jh}(T)$&size (number of agents) of the $h$-th sub-cluster of the $j$-th cluster at step $T$ of the distributed $k$-means algorithm;\\
$\sigma_i(T)$&$d\times 1$ vector selected by agent $i$ in order to calculate the centroids at step $T$ of the distributed $k$-means algorithm;\\
$\epsilon_i(T)$&scalar selected by agent $i$ in order to calculate the centroids at step $T$ of the distributed $k$-means algorithm;\\
$\eta_i^0(T)$&$k(d+1)\times 1$ vector used to calculate the cluster centroid at step $T$ of the distributed $k$-means algorithm;\\
$\overline \eta_j(T)$&$(d+1)\times 1$ vector used to calculate the cluster centroid at step $T$ of the distributed $k$-means algorithm;\\
$\overline \eta(T)$&$k(d+1)\times 1$ vector obtained by using the max-consensus algorithm in order to calculate the cluster centroid at step $T$ of the distributed $k$-means algorithm;\\
$\overline \sigma_j(T)$&$d\times 1$ vector used to calculate the $j$-th centroid at step $T$ of the distributed $k$-means algorithm;\\
$\overline \epsilon_j(T)$&scalar used to calculate the $j$-th centroid at step $T$ of the distributed $k$-means algorithm;\\
$h^j_{\max}(T)$&number of subclusters composing cluster $j$ at step $T$ of the distributed $k$-means algorithm;\\
$\nu_i^0(T)$&scalar initial condition used to verify the exit condition by means of consensus algorithms at step $T$ of the distributed $k$-means algorithm;\\
$\overline \nu_i(T)$&scalar variable used to verify the exit condition at step $T$ of the distributed $k$-means algorithm;\\
\end{supertabular}

\subsection{Definitions}
%In the following we will review the average-consensus problem and the max consensus problem in the discrete-time fashion and with respect to fixed topologies.

Let $G=\{\mathcal{V},\mathcal{E}\}$ be a graph, where  $\mathcal{V}$ is a set of $n$ {\it vertices} $v_1, \ldots v_n$ and $\mathcal{E}$ is the set of {\it edges} $(v_i,v_j)$ that allow a communication from vertex $v_i$ to vertex $v_j$. 
A graph is said to be {\it undirected} if $(v_i,v_j)\in \mathcal{E}$ whenever $(v_j,v_i)\in \mathcal{E}$ (i.e., the communication is always bidirectional), and is said to be {directed} otherwise.
A graph $G$ is {\it connected} if for any $v_i, v_j \in \mathcal{V}$ there is a path whose endpoints are in $v_i$ and $v_j$, without necessarily respecting the orientation of edges.
A graph $G$ is {\it strongly connected} if for any $v_i, v_j \in \mathcal{V}$ there is a path whose endpoints are in $v_i$ and $v_j$, respecting the orientation of edges.

A graph $G$ can be represented by a $n\times n$ {\it adjacency matrix} $A$, whose elements $a_{ij}=1$ if $(v_i,v_j)\in \mathcal{E}$ and are zero otherwise.
Let the {\it in-neighborhood} $\mathcal{N}_i^{in}$ of a vertex $v_i$ be the set of vertices $\{v_j : (v_j,v_i) \in \mathcal{E}\}$, while $\mathcal{N}_i^{out}$ is the {\it out-neighborhood} of vertex $v_i$, i.e.,the set of vertices $\{v_j : (v_i,v_j) \in \mathcal{E}\}$; for undirected graphs the in-neighborhood and the out-neighborhood coincide and are referred to as the {\em neighborhood} $\mathcal{N}_i$.
In the remainder of this paper, where not explicitly stated, we will assume the graphs to be connected and undirected.

In the following we denote by T the index of a step in the (distributed) $k$-means algorithm, while we denote by $t$ the index of an iteration of average-consensus or max-consensus algorithms executed within each step of the distributed $k$-means algorithm.

\section{$K$-Means Algorithm}
\label{kmeans}
Consider a set of $n$ observations $x_1, \ldots, x_n$, where each observation $x_i$ is a vector in $\mathbb{R}^d$.
Suppose we want to partition the $n$ observations into $k$ ($k\leq n$) sets or {\it clusters} $S_1, \ldots, S_k$.
Specifically, we want to find a set of centroids $c_1, \ldots, c_k$, each associated to a cluster, and we want to solve the following optimization problem:

\begin{equation}
\label{optptob}
\begin{matrix}
\min D= \sum_{i=1}^n \sum_{j=1}^k r_{ij}||x_i-c_j||^2\\
\\
\mbox{Subject to}\\
\begin{cases}
\sum_{j=1}^k r_{ij} =1 & \forall i=1, \ldots, n\\
r_{ij}\in\{0,1\}& \forall i=1, \ldots, n; \forall j=1, \ldots, k
\end{cases}
\end{matrix}
\end{equation}
where $r_{ij}=1$ if observation $x_i$ is assigned to the set $S_j$ and $r_{ij}=0$ otherwise, and $c_j\in\mathbb{R}^d$ is the centroid of the observations within the set $S_j$. 

Problem \eqref{optptob} is hard to solve exactly when $n$ and $k$ are large (it is NP-hard in general Euclidean space $\mathbb{R}^d$, even for 2 clusters \cite{aloise2009np} and for a general number of clusters $k$, even in the plane \cite{mahajan2009planar}), and in the literature  several heuristic algorithms have been proposed.
Among others, the {\it $k$-means algorithm} \cite{MacQueen} proved its effectiveness.

%For each step $T$, the algorithm minimizes an objective function: 
%\begin{equation}
%\label{funzioneobiettivo}
%D(T)= \sum_{i=1}^n \sum_{j=1}^k r_{ij}(T)||x_i-c_j(T)||\\
%\end{equation}
%where $c_j(T)$ is the current centroid of set $S_j(T)$ and $r_ij(T)$ is the current assignment variable for agent $i$ and set $j$, respectively.

The $k$-means algorithm starts with a random set of $k$ centroids $c_1(0), \ldots, c_k(0)$, and alternates at each step between an {\it assignment} and a {\it refinement} phase.

During the assignment phase, each observation $x_i$ is assigned to the set characterized by the nearest centroid, i.e.:
\begin{equation}
r_{ih}(T)=\begin{cases}
1 & \mbox{ if } h= \mbox{ argmin}_j ||x_i-c_j(T)||^2\\
0 & \mbox{ else}
\end{cases}
\end{equation}

During the refinement phase each centroid $c_j$ is updated as the centroid of the observations associated to $S_j(T)$, i.e.:
\begin{equation}
c_j(T+1)= \frac{\sum_{i=1}^n r_{ij}(T) x_i}{\sum_{i=1}^n r_{ij}(T) }
\end{equation}
The two steps are iterated until convergence or up to a maximum of $M$ iterations.

%\textcolor{red}{Q: GABRIELE, is it possible for a centroid not to include any observations? What happens in that case. I GUESS THIS IS AN ISSUE OF THE ORIFGINAL ALGORITHM, but in (3) we would have 0/0}

%\textcolor{blue}{A: Yes it is an issue of the standard k means. Do you think the comments later in the paper are enough for this issue?}

\begin{figure}[h!]
\begin{center}
\includegraphics[width=3.5in]{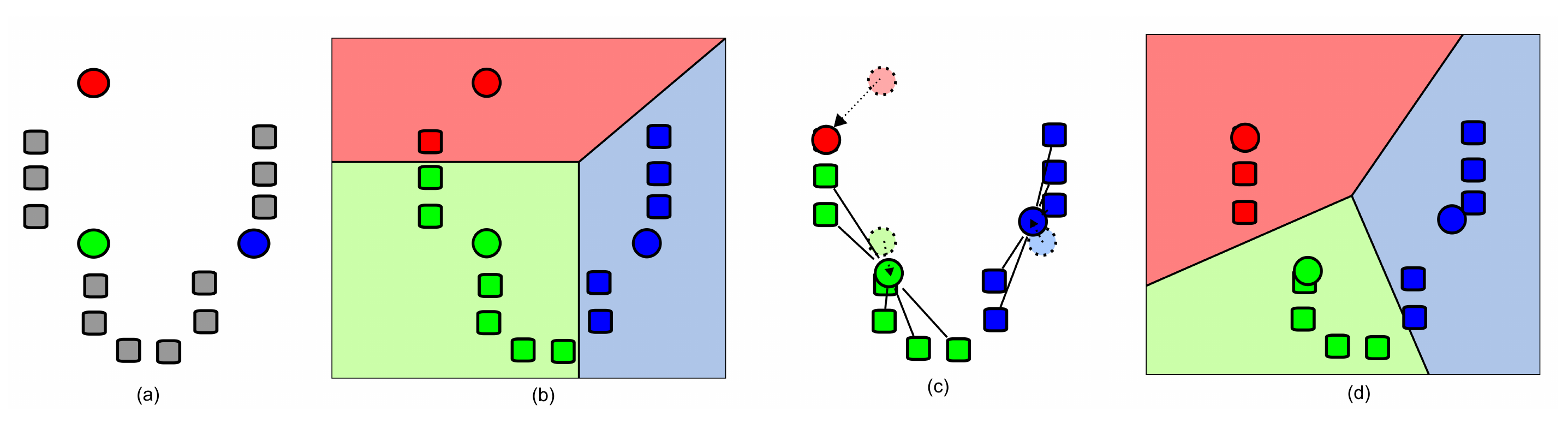}
\caption{Example of execution of $k$-means algorithm (source; Wikimedia Commons available under GNU Free Documentation License v. 1.2).}
 \label{fig:kmeans}
\end{center}
\end{figure}

Figure \ref{fig:kmeans} is an example of execution of the algorithm for a set of $n=12$ observations in $\mathbb{R}^2$ and for $k=3$.
Specifically, Figure \ref{fig:kmeans}.(a) shows with circles the initial centroids, Figure \ref{fig:kmeans}.(b) and Figure \ref{fig:kmeans}.(c) report the assignment and refinement phases for the first step, while Figure \ref{fig:kmeans}.(d) depicts the assignment phase for the second step.

%\textcolor{red}{Q: GABRIELE, WHAT ARE THE POSITIVE THINGS ABOUT THE K MEANS ALGORITHM? WHAT MAKES IT WORTHY OF STUDY. MENTION SOMETHING HERE TO BETTER MOTIVATE THE FOCUS OF THE PAPER ON THE K-MEANS ALGORITHM.}

%\textcolor{blue}{A: do you think the following paragraph is enough?}

The $k$-means algorithm is known to converge to a local optimum value, while there is no guarantee to converge to the global optimum \cite{MacQueen,Brucker:77}.
However, given the complexity of the problem at hand, $k$-means algorithm is de facto the most diffused heuristic algorithm: indeed ``ease of implementation, simplicity, efficiency, and empirical success are the main reasons for its popularity" \cite{jain2010data}. 

Since there is a strong dependency on the initial choice of the centroids, a common practice is to execute the algorithm several times, each time with different initial conditions, and select the best solution.

Note that for each step, each of the $n$ observations, and each of the $d$ components of the observations, the algorithm calculates the difference with each of the $k$ centers; hence the computational complexity is $O(d\,k\,n\,M)$ \cite{Brucker:77}.
\vspace{2mm}

\begin{remark}
\label{remark:kmeans++}
A center that is not chosen in the first step by any agent has a chance to be never chosen, e.g., when a random center is chosen far from the other centers and observations.
In this case a possible solution is to replace the center with a new random center.
%An alternative solution is to select better initial centroids: for instance, it is possible to select some of the observations at random. 
A more sophisticated approach, namely {\em $k$-means$++$} algorithm \cite{arthur2007k}, selects one initial centroid at random among the observations, and the remaining ones are chosen in order to try to maximize the inter-centroid distance. 
%Specifically,  the remaining $k-1$ centroids are chosen from the remaining observations with probability proportional to the squared distance from the observation's nearest centroid.
As the density of observations grows, however, the risk of this singular case becomes less and less likely.
\end{remark}

\begin{remark}
The observations can be scaled in order to weigh the different components and obtain a clustering where some components have more importance than others.
This can be done by using the following norm:
\begin{equation}
||x||_A=<Ax,\,Ax>
\end{equation}
where A is a $d\times d$ diagonal matrix whose entries $a_{ii}$ represent the weights of the different components. As a result the following function is minimized instead of the one in Eq. \eqref{optptob}:

\begin{equation}
\label{funzioneobiettivoscalata}
D(T)=\sum_{i=1}^n \sum_{j=1}^k r_{ij}(T) ||x_i-c_j(T)||^2_A.
\end{equation}

\end{remark}

%\begin{remark}
%\label{exitremark}
%There is no general rule, in principle, for the choice of the maximum number of iterations $M$.
%
%\textcolor{red}{Q: GABRIELE, ARE THERE ANY RESULTS STATING THAT THE K MEANS ALGORITHM WILL CONVERGE if M is large? Or giving conditions under whic it will converge? Or bounding its performance away from the optimal? If so, it might be interesting to mention them.}
%
%\textcolor{blue}{A: I added the following discussion, do you think it is enough?}
For what concerns the termination of the algorithm, to the best of our knowledge there is no available result about a bound for the number of iterations $M$.
Notice that the $k$-means algorithm can show slow convergence on particular instances \cite{vattani2011k}.
Such particular cases do not seem to arise in practice, as witnessed by the fact that the $k$-means algorithm has smoothed polynomial complexity \cite{arthur2009k}.

A possible practical solution to terminate the algorithm before the maximum number of iterations $M$ (which can be quite high) is reached, is when one of the following conditions is verified:
\begin{enumerate}
\item $M$ iterations are completed
\item there is no change in the assignment variables $r_{ij}(T)$ with respect to variables $r_{ij}(T-1)$;
\item $D(T)-D(T-1)<\Delta_{\max}$, for some $\Delta_{\max}>0$.
\end{enumerate}

\section{Related Work and Contribution}
\label{relwork}
%\textcolor{blue}{I tentatively moved this section here. Also the title of section and subsections is tentative.}
\subsection{Related Work}
In the literature, parallel and distributed $k$-means implementations have been motivated by the need to categorize huge data sets of high-dimensional observations. 
In \cite{Ng:2000,Dhillon:2000,Zhang:2006} the focus is mainly on multiprocessor architectures; in \cite{Vendramin:2011} a parallel implementation of fuzzy clustering techniques is provided; in \cite{Patel:2013} a parallel version of $k$-means is provided with a specific focus on privacy preserving aspects.

%\textcolor{red}{Q:These earlier works,probably did not consider the cost of wireless communication? Because in this earlier work most of the concern focuses on computation?}

%\textcolor{blue}{A: the focus of these papers is on parallel computing or on having distributed decisors, each with all the information and choosing the best solution (so it is like n instances of the centralized algorithm)}.

Another typical distributed application involving clustering is the so called {\it consensus clustering}, 
where a set of agents each execute a clustering algorithm over the same data set. 
The agents then must reach a consensus on the optimal partition \cite{Monti:2003,Topchy:2004,Goder:2008,Forero:2008}.

In the context of wireless sensor networks there are several attempts to implement distributed clustering schemes, especially with respect to probabilistic approaches: in \cite{23} an incremental scheme is given; in \cite{18} the distributed approach is based on a gossip algorithm; in \cite{14} distributed consensus algorithms are used to reach consensus on a Mixture of Gaussian distribution; a methodology that requires the graph to be a tree is provided in \cite{29}, and a scheme based on the method of multipliers is given in \cite{12}; in \cite{Chen:2004} a decentralized clustering algorithm is provided for a hierarchical sensor network, composed of a static backbone of cluster heads and a set of low-end sensors; in \cite{Younis:2004} a distributed clustering algorithm for a sensor network is given assuming agents with different power levels.
In \cite{Forero:2008}  distributed clustering schemes are provided for a sensor network considering a consensus clustering approach.

The above works, unfortunately, are limited to a parameter estimation or to specific network topologies and configurations, or require each agent to handle an amount of information which is comparable to the centralized approach.

\subsection{Contribution}
In this paper, we assume that each agent has access to some (possibly high-dimensional) measurements (e.g., position, temperature, humidity, etc.).

We provide a distributed algorithm to cluster such an information when the agents exchange information exclusively with their one-hop neighbors.
Moreover, in spite of previous works, here we assume that each agent has very limited information about of the structure and the topology of the network.

Specifically, each agent 
\begin{enumerate}
\item has a unique identifier;
\item can exchange directly information only with his one-hop neighbors;
\item can act synchronously.
\item knows an upper bound $\tilde n$ for the number $n$ of agents in the network.
\end{enumerate}

Moreover, the network is described by a graph $G$. Consequently, any two agents can use a multi-hop strategy to exchange information, as will be discussed later in the paper.

According to the circumstances, there may be the need to group the agents into clusters that are homogeneous according to the agents' positions or to the sensor measurements.
If this task has to be performed by agents without any centralized supervision, we need to implement a distributed $k$-means clustering algorithm.

Each agent, therefore, has to solve an assignment problem with respect to a set of centroids (i.e., finding the nearest one) as well as to contribute to the refinement of the centroids, interacting with the other agents.

The procedure is executed by means of one-hop communication and the result is proven to be the same as the centralized $k$-means algorithm in terms of minimization of the objective function.
This is done by resorting to a combination of distributed consensus algorithms such as average-consensus \cite{Olfati1,Olfati2} and max-consensus \cite{cortes2008distributed}.

\label{distributedkmeans}
\begin{figure}[h!]
\begin{center}
\includegraphics[height=3in]{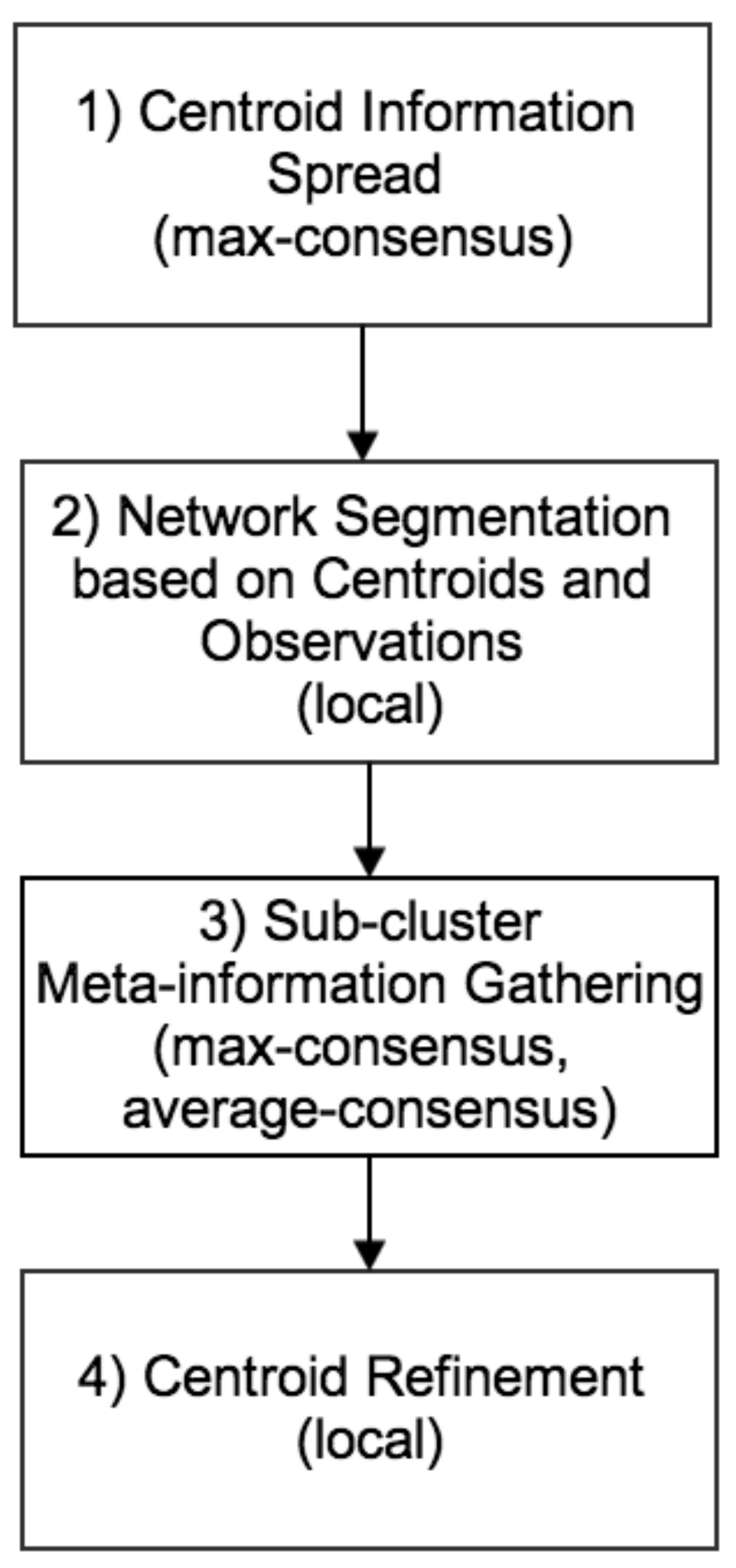}
\end{center}
\caption{Conceptual flow-chart of one step of the proposed distributed $k$-means algorithm.}
 \label{fig:conceptual}
\end{figure}

Figure \ref{fig:conceptual} shows a conceptual scheme of one step of the proposed algorithm.
Specifically, the proposed algorithm takes advantage of a max-consensus algorithm for spreading the information on current centroids across the network (phase 1 in Figure \ref{fig:conceptual}). 
Such an information is then used locally by each agent to select the nearest centroid, thus segmenting the network into communities (i.e., clusters) (phase 2). 
Notice that there is, in general, no guarantee that the communities obtained during phase 2 will be connected, and therefore, in order to update the centroids, there is the need to gather meta-information on the sub-clusters that compose each cluster, i.e., their size and their centroid (phase 3); this is done by means of a combination of max-consensus and average-consensus algorithms.
Once such an information has been gathered, the agents are able to update the centroids in a local manner (phase 4).

A typical application scenario of the proposed methodology is a wireless setting where each agent is equipped with a low range communication device, and is able to share information only with those agents that fall in its communication range; however the proposed approach is more general and applies to any connected undirected graph; moreover, it applies also to strongly connected and balanced graphs, and is easily extended to those directed graphs which can be balanced using weight balancing techniques like those presented in \cite{Hadjicostis:2012,Cortes:2012,Cai:2012}. 
For the sake of simplicity, however, we limit the scope of the paper to undirected graphs.

%\textcolor{red}{Q: By the way, a digraph can be balanced as long as it is strongly connected (specifically, a digraph ca be balanced is and only if it is a collection of strongly connected diagraphs). I AM NOT SURE WHAT THE REVIEWER HAD INMIND. ALSO NOTE THAT THERE IS FINITE TIME ALGORITHMS FOR DIAGRAPHS.}
%\textcolor{blue}{A: If you agree I'd ignore this point}

%Finally we consider several performance indices for the automatic selection of the optimal value of the parameter $k$ in a user defined range of values, and provide a distributed algorithm for the optimal selection of such parameter.

\section{Distributed Consensus Algorithms}
\label{consensus}
In order to provide a distributed implementation of the $k$-means algorithm, we need a mechanism to diffuse information among the agents and to refine the centroid of a subset of agents in the network.

The {\it max-consensus} \cite{cortes2008distributed} and {\it average-consensus} \cite{Olfati1,Olfati2} algorithms proved their effectiveness in composing local observations by means of one-hop communication.
Let a set of $n$ agents, each associated to a vertex in the graph $G$ and a discrete-time dynamic equation in the form:
\begin{equation}
z_i(t+1)=u_i(\mathcal{N}_i\cup \{i\},t)
\end{equation}
where $u_i(\mathcal{N}_i\cup \{i\},t)$ is a function of the state of agent $i$ and of the agents $j$ that belong to the neighborhood of agent $i$ and $z_i(t)\in \mathbb{R}^d$ for all $i$.

Let $\chi(z_{10}, \ldots, z_{n0})\in \mathbb{R}^d$ be any function of the initial conditions of all the agents; the $\chi$-consensus problem consists in finding an input $u_i(\mathcal{N}_i\cup \{i\},t)$ such that $\lim_{t\rightarrow \infty} z_i(t)=\chi(z_{10}, \ldots, z_{n0})$ for all $i=1, \ldots n$.

Among several different problems that can be solved via consensus algorithms, we next review the max-consensus and average-consensus problems.

\subsection{Max-consensus}
In the $\max$-consensus problem, the agents have to converge to the maximum of the initial conditions, i.e., $\chi(\cdot)$ is the component-wise maximum of its arguments. 
The problem, for connected  undirected graphs or strongly connected directed graphs, is known to have a solution in finite time \cite{cortes2008distributed} if the following control law is chosen:
\begin{equation}
\label{max_consensus}
u_i(\mathcal{N}_i\cup \{i\},t)=\max_{j\in \mathcal{N}_i\cup \{i\}} z_j(t)
\end{equation}

With the above control law, the problem is solved in $m\leq n$ iterations, where $m$ is the length of the diameter of the graph (i.e., the maximum among the minimum paths for every possible pair of agents, respecting the orientations of the links if the graph is directed).
Again, since the agents do not know $m$ nor $n$, for practical applications a maximum number of iterations $t_{\max}>n$ has to be selected.

In the following we will denote by

$$\overline z_i= \mbox{max-consensus}_i(z_i(0),z_j(0)|\,\, j\neq i,G,t_{\max})$$
the execution of $t_{\max}$ iterations of the max-consensus procedure by the $i$-th agent in a network $G$, starting from its own initial condition $z_i(0)$ and the ``unknown" initial conditions of the other agents, while $\overline z_i$ is the state of the $i$-th agent at iteration $t_{\max}$. Such a formalism just represents the execution of the max-consensus by the $i$-th agent, and we assume that all other agents are executing the same algorithm in a synchronous manner, each with its own initial condition.

%\textcolor{red}{Q: You could change that to something that includes all initial values, e.g., as I have attempted to do below (perhaps check Nancy Lynch's book to see what would be standard notation for this.}
%
%\textcolor{blue}{A: I was stressing that each agent has its own initial value and communicates only with neighbors. I think that the notation $\mbox{max-consensus}_i$ is much more clear. I don't know about the initial values of other agents, did you mean something like this? 
%$$z_i(t_{\max})= \mbox{max-consensus}_i(z_j(0)| j\in \mathcal{N}_i \cup \{i\},G,t_{\max})$$
%But in this way it seems that each agent knows each other initial condition (maybe I did not understand what you had in mind for this point)
%}

\subsection{Average-consensus}
In the average-consensus problem the agents are required to converge to the component-wise average of their initial conditions.

Let each agent choose
\begin{equation}
\label{linear_iteration}
u_i(\mathcal{N}_i\cup \{i\},t)=w_{ii} z_i(t) + \sum_{j=1}^n w_{ij} z_j(k)
\end{equation}
where $w_{ij}=0$ if $(v_i,v_j)\notin \mathcal{E}$.
The update strategy for the entire system can be represented as

$$
z(t+1)=(W\otimes I_d) z(t)
$$
where matrix $W$ contains the terms $w_{ij}$, $I_d$ is the $d\times d$ identity matrix and $z(t)$ is the stack vector containing the terms $z_{i}(t)$.

According to \cite{xiao2004fast}, with this choice of $u_i(\mathcal{N}_i\cup \{i\},t)$ the state of each agent asymptotically converges to the component-wise average of the initial states if and only if: 
\begin{enumerate}
\item $W$ has a simple eigenvalue at $1$  and all other eigenvalues have magnitude strictly less than $1$; \item the left and right eigenvectors of $W$ corresponding to eigenvalue $1$ are $1_n$ and $\frac{1}{n}1_n$, respectively, where $1_n$ is a vector composed of $n$ elements, all equal to one.
\end{enumerate}

%A possible choice  \cite{Olfati1} for the terms $w_{ij}$ is 
%$$
%w_{ij}=
%\begin{cases}
%1-\tau \#(\mathcal{N}_i)& \mbox{if } v_j \in \mathcal{N}_i \\
%0& \mbox{if } v_j \not\in \mathcal{N}_i \\
%\tau & \mbox{ if } i=j
%\end{cases}
%$$
%
%where $\tau$ is a constant parameter which is the same for all the agents, and must satisfy
%$$
%\tau< \frac{1}{\min_{i=1,\ldots,n} \#(\mathcal{N}_i)}. 
%$$

A possible choice, assuming that the underlying graph is undirected and connected and that each agent knows $n$ (or an upper bound for $n$), is that each agent $i$  chooses independently the terms $w_{ij}$ as
$$
w_{ij}=\begin{cases}
\frac{1}{n}, & \mbox{if } v_j\in \mathcal{N}_i\\
0, & \mbox{if } v_j\notin \mathcal{N}_i\\
1-\sum_{l=1}^l w_{il}, & \mbox{if } i=j
\end{cases}
$$
resulting in a matrix $W$ that satisfies the conditions in \cite{xiao2004fast}. Several other choices that yield asymptotic consensus are possible (e.g., see \cite{ren2005survey}).

\vspace{2mm}

\begin{remark}
\label{remark:stima_N}
Combining the max-consensus and the average-consensus algorithm, it is possible to calculate the number of agents $n$ in the network in a distributed way \cite{shames2012distributed}.
Specifically, suppose a {\em leader} is elected via max-consensus (e.g., each agent has an identifier and the agent with the maximum identifier is elected as leader via max-consensus).
Now, let the agents execute a scalar (i.e., $d=1$) average-consensus algorithm with $z_i(0)=1$ if agent $v_i$ is the leader and $z_i(0)=0$ otherwise: then, average-consensus yields for all $i=1, \ldots, n$
$$
\lim_{t\rightarrow \infty} z_i(t) = \frac{1}{n}.
$$
 
\end{remark}

Regarding the computational complexity of both max and average-consensus algorithms, for each agent $i$ we have that  $\#(\mathcal{N}_i)\leq n$, hence for vectorial initial conditions $z_i(0)\in\mathbb{R}^d$, the complexity is  $O(dn\,t_{\max})$ for each agent, where $t_{\max}$ is the total number of iterations done.

Knowing an upper bound $\tilde n$ for $n$, it is possible to choose $t_{\max}=\tilde n$ in the case of max-consensus; in the average-consensus algorithm, the solution is reached asymptotically, hence to obtain a good approximation of the solution the agents stop after a large number of iterations or when particular conditions are met \cite{yadav2007distributed,manitara2014distributed}.

The next subsection is devoted to reviewing a finite time average-consensus algorithm that overcomes such an issue.

\subsection{Finite Time Average-consensus}
\label{subsec:finite_time_consensus}

%In the literature several extensions of the above control law have been provided to grant convergence in finite time: among the others let us report the following by Wang and Xiao \cite{Wang:2007}:
%\begin{equation}
%\label{average_finite}
%u_i(\mathcal{N}_i\cup \{i\},t)=\tau \sum_{j\in \mathcal{N}_i}sign(z_j(t)-z_i(t))|(z_j(t)-z_i(t)|^{\alpha_{ij}}
%\end{equation}
%where $0<\alpha_{ij}<1$ and the operations $sign(\cdot)$, $|\cdot|$ and $(\cdot)^{\alpha_{ij}}$ have to be intended component-wise. The above control law solves the problem in finite time if the matrix containing the terms $\alpha_{ij}$ is symmetric.
%
%For the average computation to be fully distributed, each agent has to be provided with a stop criterion. In fact, even if a finite time consensus protocol is chosen, the agents do not typically know such time, which is dependent on how much the graph $G$ is connected, hence on global information.
%A possible solution is that each agent, while continuing to execute the consensus algorithm for exactly $t_{\max}>>1$ iterations, stores a copy of $z_i(t)$ which is updated until $|z_i(t)-z_i(t-1)|<\epsilon$, with $\epsilon << 1$.

In the literature, several extensions of the average-consensus algorithm have been provided to grant convergence in finite time \cite{Wang:2010,sundaram2007finite,6161213,charalambous2013decentralised,tran2013distributed}, as well as some distributed stopping approaches that allow the agents to detect when they have converged within some $\epsilon$ of the asymptotic consensus value \cite{yadav2007distributed,manitara2014distributed}.
In this paper we will resort to the schema proposed in \cite{sundaram2007finite}, because, knowing an upper bound $\tilde n$ of $n$, the number of steps required to reach consensus phase is known.

Suppose the agents perform an average-consensus algorithm using the update rule $u_i(\mathcal{N}_i\cup \{i\},t)$ of Eq. \eqref{linear_iteration}; in the following we will describe the algorithm with respect to a scalar setting (i.e.,  $d=1$), as the extension to vectorial states is trivial.
The approach in \cite{sundaram2007finite} allows each agent $i$ to calculate the consensus value as a linear combination of $z_i(0),\ldots,z_i(\delta_i)$ for a finite $\delta_i>0$ which may not be the same for each agent.
In other terms, each agent $i$ seeks a coefficient vector 

$$
\Gamma_i^{'}=\begin{bmatrix}\gamma_{i,\delta_i}&\gamma_{i,\delta_{i-1}}&\ldots&\gamma_{i,0}\end{bmatrix}
$$
such that 
$$
\begin{bmatrix}z_i(\delta_i)&z_i(\delta_i-1)&\ldots&z_i(0)\end{bmatrix}\Gamma_i=\frac{1}{n}1_n^{'}z(0).
$$

%Notice that
%$$
%z_i(t)=e_i^{'}z(t)=e_i^{'}W^tz(0)
%$$
%where $e_i$ is the $i$-th vector of the canonical basis, hence
%$$
%e_i^{'}\Big(\gamma_{i,\delta_i}W^{\delta_i}+\gamma_{i,\delta_i-1}W^{\delta_i-1}+\ldots+\gamma_{i,0}I_N
%\Big)=\frac{1}{n}1_n^{'}
%$$

Let $q_i(g)$ be the $i$-th {\em minimal polynomial}, i.e., the monic polynomial of smallest degree such that $e_i^{'}q_i(W)=0$, where $e_i$ is the $i$-th vector of the canonical basis. If the conditions $1)$ and $2)$ from \cite{xiao2004fast} are verified, then $q_i(g)$ has a root in $g=1$ with multiplicity $1$ \cite{sundaram2007finite}. 

The minimal polynomial $q_i(g)$ can be expressed as

$$
q_i(g)=g^{\delta_i+1}+\alpha_{i,\delta_i}g^{\delta_i}+\alpha_{i,\delta_i-1}g^{\delta_i-1}+\ldots+\alpha_{i,0}.
$$

Let the polynomial 
$$p_i(g)=\frac{q_i(g)}{g-1}=g^{\delta_i}+\beta_{i,\delta_i-1}g^{\delta_i-1}+\ldots+\beta_{i,0}.
$$

The terms $\gamma_{i,j}$ for $i=1,\ldots, n$, $j=0, \ldots, \delta_i$ are calculated as:

$$
\gamma_{i,j}=\frac{\beta_{i,j}}{n\,p_i(g)\Big|_1}%\frac{\frac{1}{n}1_n^{'} 1_n}{p_i(1)e_i^{'}1_n}\beta_{i,j}
$$
where $p_i(g)\Big|_1$ is the value of $p_i(g)$ evaluated at $g=1$ and we take $\beta_{i,\delta_i}=1$.

As for the calculation of the coefficients $\beta_{i,j}$ of $q_i(g)$, notice that, since $e_i^{'}q_i(W)=0$ and $z(t)=W^tz(0)$ for all integers $t\geq 0$, it holds

\begin{equation}
\label{eq:13}
e_i^{'} \Big(
W^{\delta_i+1}+\alpha_{i,\delta_i}W^{\delta_i}+\alpha_{i,\delta_i-1}W^{\delta_i-1}+\ldots+\alpha_{i,0}I_N\Big)=0
\end{equation}
or equivalently
$$
\begin{bmatrix}
\alpha_{i,0}&\ldots&\alpha_{i,\delta_i-1}&\alpha_{i,\delta_i}&1
\end{bmatrix}\Theta_i=0
$$
where
$$
\Theta_i=\frac{1}{n}
\begin{bmatrix}
1_n^{'}\\
1_n^{'}W\\
\vdots\\
1_n^{'}W^{\delta_i+1}\\
\end{bmatrix}
$$

The polynomial $q_i(g)$ can be easily constructed in a centralized fashion noting that $\Theta_i$ is the observability matrix of the pair $(W,e_i^{'})$, hence $q_i(g)$ can be found by forming the matrix $\Theta_i$ and increasing $\delta_i$ until $\Theta_i$ loses rank. The coefficients of $q_i(g)$ can be obtained from the left nullspace of $\Theta_i$.

Assuming that the agents know an upper bound $\tilde n$ of $n$, a distributed procedure to let each agent construct $q_i(\cdot)$ is as follows.
Let $z_{*,1}(0),\ldots, z_{*,\tilde n}(0)$ denote a set of $\tilde n$ different initial condition vectors for all the agents. 
For each of such initial conditions, each agent executes the update rule of Eq. \eqref{linear_iteration} for $\tilde n+1$ iterations (i.e., they calculate $z_{*,j}(t+1)=Wz_{*,j}(t)$ for $j=1,	\ldots, \tilde n$ and $0\leq t \leq \tilde n-1)$.
Each agent $i$ has then access to the $\tilde n\times (\delta_i+2)$ matrix

\begin{equation}
\label{eq:15}
Z_{i,\delta_i}=\begin{bmatrix}
z_{i,1}(\delta_i+1)&z_{i,1}(\delta_i)&\cdots & z_{i,1}(0)\\
z_{i,2}(\delta_i+1)&z_{i,2}(\delta_i)&\cdots & z_{i,2}(0)\\
\vdots&\vdots&\ddots&\vdots\\
z_{i,\tilde n}(\delta_i+1)&z_{i,\tilde n}(\delta_i)&\cdots & z_{i,\tilde n}(0)
\end{bmatrix}
\end{equation}
for any $0\leq \delta_i \leq \tilde n-1$.

Suppose $\delta_i$ is the smallest integer such that $Z_{i,\delta_i}$ is not full column rank. Then there is a vector 
$$\alpha_i=
\begin{bmatrix}
1&\alpha_{i,\delta_i}&\alpha_{i,\delta_i-1}&\ldots&\alpha_{i,0}
\end{bmatrix}^{'}
$$
such that $Z_{i,\delta_i}\alpha_i=0$, implying that
$$
z_{i,j}(\delta_i+1)+\alpha_{i,\delta_i}z_{i,j}(\delta_i)+\ldots+\alpha_{i,0}z_{i,j}(0)=0.
$$
Note that $z_{i,j}(t)=e_i^{'}z_{*,j}(t)=e_i^{'}W^tz_{*,j}(0)$, hence it holds
$$
e_i^{'}\Big(
W^{\delta_i+1}+\alpha_{i,\delta_i}W^{\delta_i}+\ldots+\alpha_{i,0}I
\Big)z_{*,j}(0)=0
$$
for $j=1,2,\ldots,n$. If $z_{*,1}(0),\ldots, z_{*,n}(0)$ are linearly independent, the above expression is equivalent to that of Eq. \eqref{eq:13}. This means that the degree of $q_i(g)$ can be determined from the smallest integer $\delta_i$ for which matrix \eqref{eq:15} loses rank, and the nullspace of the corresponding matrix provides the coefficients of $q_i(g)$.

\begin{remark}
There are various ways to ensure the linear independence of the $\tilde n$ different initial conditions. For example, if all initial conditions are taken to be independent, identically distributed random variables with a Gaussian distribution, then the vectors $z_{*,1}(0),\ldots, z_{*,\tilde n}(0)$ will almost surely be linearly independent. Alternatively, if the agents have unique identifiers, one can artificially choose the initial conditions to impose linear independence. For instance, during the $j$-th run, the agent with the $j$-th greatest identifier can choose its initial condition to be 1, whereas all other agents set their initial conditions to be zero.
Notice that several max-consensus algorithms are required to choose the agents with greatest identifiers.
\end{remark}
 
 \vspace{2mm}
 
%\textcolor{blue}{Can we use the fact the agents have unique identifiers to find the above initial conditions without doing several max-consensuses?}
 
% \textcolor{red}{Q: Also agent that this setting up of the problem might be done only once (in case the k means prouder will be performed multiple times, e.g. When the same sensor network provides measurements at multiple times of the day)}
% 
%\textcolor{blue}{A: that is interesting, but maybe I did not understand well. Do you mean that the weights used to calculate the average depend just on the topology (and not on the specific initial conditions) and that we can calculate them once for all?} 
 
Regarding the computational complexity, for each agent the update rule is executed exactly $(\tilde n+1)^2$ times, hence for vectorial initial conditions $z_i(0)\in\mathbb{R}^d$, the complexity is  $O(dn^2)$ for each agent. Such a complexity is comparable with that of max-consensus algorithm, and it is in many cases smaller than the complexity of standard average-consensus algorithm, where $t_{\max}$ is typically big.
 
 In the following, with a slight abuse of notation, we will denote by
$$\overline z_i= \mbox{FTA-consensus}_i(z_i(0), z_j(0)|\,\, j\neq i,G,\tilde n)$$
the execution of the finite time average-consensus procedure by the $i$-th agent in a network $G$, starting from the initial opinion $z_i(0)$, while $\overline z_i$ is the state of the $i$-th agent at the end of the algorithm and $\tilde n$ is the upper bound of $n$ used within the algorithm. Again, we assume that all other agents are executing the same algorithm in a synchronous manner, each with its own initial condition.
 
We will also denote by 
$$
n_i= \frac{1}{\mbox{FTA-consensus}_i(n_i^0,z_j(0)|\,\, j\neq i,G,\tilde n)}
$$ 
the result of the distributed calculation, performed by the $i$-th agent, of the number of agents $n$ belonging to the network, using the approach in \cite{shames2012distributed} but employing the finite time average-consensus algorithm in \cite{sundaram2007finite}, where $n_i^0=1$ if $i$ is the leader elected via max-consensus and $n_i^0=0$ otherwise.%, and $n_i$ is the value of $n$ as calculated by agent $i$.
  
\section{Distributed $k$-means Algorithm}
\label{section4:distrib}
\begin{figure}[h!]
\begin{center}
\includegraphics[height=8.6in]{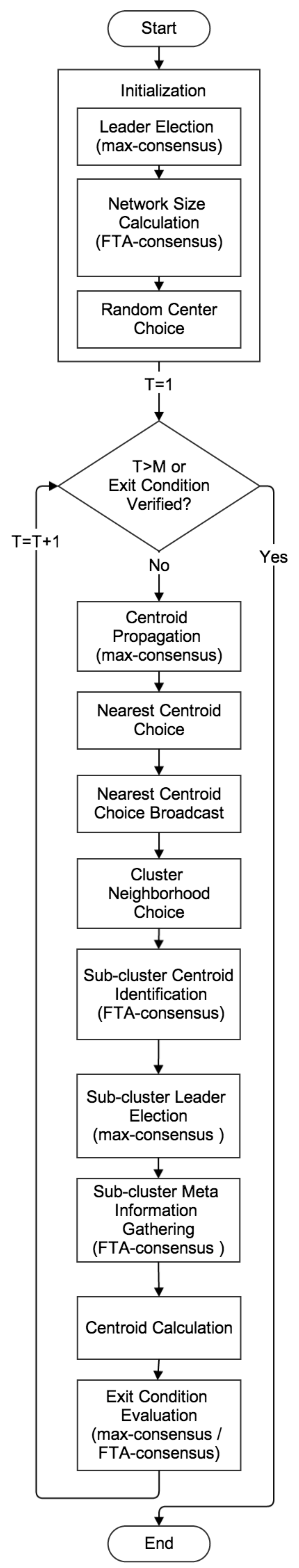}
\caption{Flow-chart of the proposed distributed $k$-means algorithm.}
 \label{fig:fixalgo}
\end{center}
\end{figure}

%\lipsum[1]
\begin{algorithm*}[ht!]
\begin{multicols}{2}
%\begin{algorithmic} [1]
%z_i(t_{\max})= \mbox{max-consensus}(z_i(0),G,t_{\max})
%\overline z_i= \mbox{FTA-consensus}(z_i(0),G,\tilde n)
%\dontprintsemicolon
\SetKw{Set}{{\bf Set}}
\caption{Distributed $k$-means algorithm (executed synchronously by all agents)}
\label{distributedkmeansalgo}
\BlankLine
\KwData{Identifier $i$, upper bound $\tilde n$, observation $x_i$, Neighborhood $\mathcal{N}_i$, Exit method $C_1$ or $C_2$, Data from neighbors}
\KwResult{Centroids $c_1, \ldots, c_k$, centroid choice $\overline j_i$, final neighborhood $\mathcal{N}^c$}
\BlankLine
\tcc{Initialization}
$c_{ij}(0)=-[\infty, \ldots, \infty]^{'}$ for all $j=1, \ldots, k$;\\ \BlankLine
$\mathcal{N}_i^c(0)=\mathcal{N}_i$;\\ \BlankLine
$\overline j_i(0)=0$;\\
\BlankLine
\tcc{Leader Election}
$i^*=\mbox{max-consensus}_i(i,j|\,\, j\neq i,G,\tilde n)$;\\
\BlankLine
\tcc{Network Size Calculation}
$$
n_i^0=\begin{cases}
1 & \mbox{if } i=i^*\\
0 & \mbox{else}
\end{cases};
$$
$n_i=1/\mbox{FTA-consensus}_i(n_i^0,j|\,\, n_j^0\neq i,G,\tilde n);$
\BlankLine
\tcc{Random Center Choice}
\If{$i = i^*$}{
	agent $i$ is the leader;\\
	$c_{ij}(0)=\mbox{random}(d\times 1)$ for all $j=1, \ldots, k$;\\
}
\BlankLine
\tcc{Main Cycle}
\BlankLine
$T=1$;\\
$exit=$false;\\
\While{$T\leq M$ and $exit=$false}{
\tcc{Centroid Propagation}	
\begin{footnotesize}
$$
c^0_i(T)=\begin{cases}
%[0, \ldots, 0]^{'} & \mbox {if }t=1 \mbox{ and } i\neq i^*\\
[c_{i1}(0)^{'}, \ldots, c_{ik}(0)^{'}]^{'} & \mbox {if }T=1\\% \mbox{ and } i=i^*\\
\mu_i(T-1) \otimes c_{i\overline j_{i}}(T-1)+ \hat \mu_i(T-1) \otimes 1_d  & \mbox{ if } T>1
\end{cases}
$$
\end{footnotesize}
$c_{i}(T)= \mbox{max-consensus}_i\,(c^0_i(T),c^0_j(T)|\,\, j\neq i,G, n_i)$;\\
\BlankLine	
\tcc{Nearest Centroid Choice}	
$$
\overline j_{i}=\arg\min_{j=1}^k |c_{ij}(T)-x_i|;
$$$$
\mu_{ij}(T)=\begin{cases}
1 & \mbox {if } j=\overline j_{i}\\
0 & \mbox {else}
\end{cases},
\quad \hat \mu_{ij}(T)=\begin{cases}
0 & \mbox {if } j=\overline j_{i}\\
-\infty & \mbox {else}
\end{cases};
$$
\BlankLine	
\tcc{Nearest Centroid Choice Broadcast}
each agent provides $\overline j_{i}$ to the neighbors in $\mathcal{N}_i$;\\
\BlankLine	
\pagebreak
\tcc{Cluster Neighborhood Choice}
each agent selects $\mathcal{N}_i^c(T)\subseteq \mathcal{N}_i$ based on $\overline j_{j}$ for each $j\in \mathcal{N}_i$;\\
\BlankLine
\BlankLine
\BlankLine 
}		
\Indp 
{
\pagebreak
\begin{small}\tcc{Sub-cluster Centroid Identification}\end{small}
$SCC_{\overline j_i \overline h}(T)=\mbox{FTA-consensus}_i(x_i,x_j|\,\, j\neq i,G^c(T), n)$
\tcc{Sub-cluster Leader Election}
$i^*_{\overline j_i \overline h}=\mbox{max-consensus}_i(i,j|\,\, j\neq i,G^c(T), n)$;\\
\BlankLine
\If{$i = i^*_{\overline j_i \overline h}$}{ \BlankLine
	agent $i$ is a sub-cluster leader;\\ \BlankLine
}
\BlankLine
\tcc{Sub-cluster Size Calculation}
\begin{small}
$$SCS_{\overline j_i \overline h}(T)=$$
$$=\frac{1}{\mbox{FTA-consensus}_i(SCL_i(T),SCL_j(T)|\,\, j\neq i,G^c(T), n)};$$
\end{small}
\BlankLine
\begin{small}\tcc{Sub-Cluster Meta Information Gathering}\end{small}
$$
\sigma_i(T)=\begin{cases}
SCC_i(T) SCS_i(T) & \mbox{if } i = i^*_{\overline j_i \overline h}\\
0_d & \mbox{else}
\end{cases};
$$
$$
\epsilon_i(T)=\begin{cases}
SCS_i(T) & \mbox{if } i = i^*_{\overline j_i \overline h}\\
0 & \mbox{else}
\end{cases};
$$
$$\eta_i^0(T)=e_{\overline j_{i}} \otimes \begin{bmatrix}\sigma_i(T)\\ \epsilon_i(T)\end{bmatrix};$$
$$
\overline \eta(T) = \mbox{FTA-consensus}_i(\eta_i^0(T),\eta_j^0(T)|\,\, j\neq i,G,n);
$$
$$
\begin{bmatrix}\overline \sigma_j(T)\\ \\ \overline \epsilon_j(T)\end{bmatrix}=\overline \eta_j, \forall j=1,\ldots,k;
$$
\BlankLine
\tcc{Centroid Calculation}
$$
c_{ij}(T)=\frac{\overline \sigma_j(T)}{\overline \epsilon_j(T)}, \forall j=1,\ldots,k;
$$
\BlankLine
\tcc{Exit Condition Evaluation}	
$$
\nu_i^0(T)= \begin{cases}
0 & \mbox{if } \overline j_{i}(T)=\overline j_{i}(T-1)\,\, \mbox{and } C_1\\
1 & \mbox{if } \overline j_{i}(T)\neq \overline j_{i}(T-1)\,\, \mbox{and } C_1\\
n||c_{i\overline j_{i}}(T)-x_i||^2& \mbox{if } C_2
\end{cases}
$$

\If{$C_1$}{
	$
	\overline \nu(T) = \mbox{max-consensus}_i(\nu_i^0(T),\nu_j^0(T)|\,\, j\neq i,G,n)
	$;\\ \BlankLine
	\If{$\overline \nu(T)=0$}{\BlankLine exit=true;\\}
}
\Else{
	$
	\overline \nu(T) = \mbox{FTA-consensus}_i(\nu_i^0(T),\nu_j^0(T)|\,\, j\neq i,G,n)
	$;\\ \BlankLine
	\If{$\overline \nu(T)-\nu(T-1)<\Delta_{max}$}{\BlankLine exit=true;\\} \BlankLine
}
%|D(T)-D(T-1)|<\Delta_{\max}

}
\Indm  
\Return  $c_{i1}(T), \ldots, c_{ik}(T)$, $\overline j_i(T)$ and $\mathcal{N}^c(T)$;\\
%\end{algorithmic}
\end{multicols}
\end{algorithm*}
%\lipsum[2]

Let a network of $n$ agents be represented by a fixed and undirected graph $G$ and suppose that each agent is able to communicate with all its one-hop neighbors in a synchronous way.
Each agent is endowed with a real vector $x_{i}\in \mathbb{R}^d$ representing a piece of information or a set of measures.
Moreover, each agent has a unique identifier and knows an upper bound $\tilde n$ of the number of agents $n$ (we assume each agent has the same value for $\tilde n$).

The objective of the distributed $k$-means algorithm is to partition the agents in $k$ clusters minimizing the functional $D$ specified in eq. \eqref{optptob} (or eq. \eqref{funzioneobiettivoscalata}) via a fully distributed approach.% involving only local interaction among neighbors. 
In the following we will neglect, without loss of generality, the procedure to mitigate the risk of obtaining a local minimum by iterating the algorithm several times and selecting the best result.
A flow chart of the proposed $k$-means distributed algorithm is provided in Figure \ref{fig:fixalgo}, while the pseudocode is given in Algorithm \ref{distributedkmeansalgo}. We assume that Algorithm \ref{distributedkmeansalgo} is executed synchronously by each agent and that the agents exchange the necessary information with their neighbors or with (a subset of) their one-hop neighborhood.

\subsection{Initialization}

The initialization phase of the Algorithm \ref{distributedkmeansalgo} is as follows.

As first step, a leader is elected. 
Supposing that each agent has a unique identifier, the agents challenge on their identifiers via max-consensus. The agent $i^*$ whose identifier is the greatest is elected as leader after no more than $\tilde n$ steps.

Based on the information on leader and non-leader agents, the size $n$ of the network is calculated resorting to the approach in \cite{shames2012distributed}, as discussed in Remark \ref{remark:stima_N}.
Notice that this operation is done by means of the finite-time average-consensus algorithm introduced in \cite{sundaram2007finite} and discussed in Section \ref{subsec:finite_time_consensus}, using the upper bound $\tilde n$. As a result of this procedure, the value of $n$ is obtained, hence for the rest of the algorithm the agents use the actual value of $n$ instead of the upper bound $\tilde n$.

To conclude the initialization phase, the leader chooses the initial centroids.
Specifically, leader $i^*$ selects $c_{i^*j}(0)\in \mathbb{R}^{d}$ at random (within the range of values of interest for each component) for all $j=1, \ldots, k$. 

The non-leader agents, conversely, choose for all $j=1, \ldots, k$
$$
c_{ij}(0)=-[\infty, \ldots, \infty]^{'}, \quad \forall i\neq i^*.
$$

 \begin{remark}
The initial centroid choice can be substituted by other approaches, such as choosing random observations  or using the $k$-means$++$ algorithm, as discussed in Remark \ref{remark:kmeans++}.
The first approach is easy to implement in a distributed way by letting each agent select a random value and then elect $k$ leaders by iterating max-consensus procedures.
The second approach requires each agent to know the current centroids and to calculate the distance between the nearest centroid and its observation, which can be implemented in a distributed fashion in a way similar to the centroid propagation procedure described in Section \ref{subsubsec:centroid_propagation}. Then, to chose the next centroid, the leader selects a random number, which is propagated via max-consensus; eventually, the agents challenge on the absolute value of the difference between the random number and the distance, again via max-consensus.
In conclusion, both of the above variations of the standard $k$-means algorithm can be implemented in a distributed fashion; in the following, for the sake of simplicity, we will consider the standard $k$-means algorithm.
 \end{remark}
 
\subsection{Main Cycle}
After the initialization phase, the main cycle is executed $M$ times, or until an exit condition is verified, as better explained in Section \ref{subsubsec:exitcond}.

At every step $T$, the main cycle is composed of the following phases: 
\subsubsection{Centroid Propagation}
\label{subsubsec:centroid_propagation}
At each step $T$, the information on the current centroids is diffused to all the agents.
 
Specifically, each agent selects a vector $c^0_i(T)\in\mathbb{R}^{kd}$ as follows:
\begin{small}
\begin{equation}
\label{eq_centroid_propagation}
c^0_i(T)=\begin{cases}
%[0, \ldots, 0]^{'} & \mbox {if }t=1 \mbox{ and } i\neq i^*\\
[c_{i1}(0)^{'}, \ldots, c_{ik}(0)^{'}]^{'} & \mbox {if }T=1\\% \mbox{ and } i=i^*\\
\mu_i(T-1) \otimes c_{i\overline j_{i}}(T-1)+ \hat \mu_i(T-1) \otimes 1_d  & \mbox{ if } T>1
\end{cases}
\end{equation}
\end{small}
where $c_{i\overline j_{i}}(T-1)\in\mathbb{R}^d$ is the centroid chosen by agent $i$ at step $T-1$ and $ \mu_i(T-1), \hat  \mu_i(T-1)\in \mathbb{R}^k$ are vectors representing the choice of a centroid at step $T-1$.
The expressions for $ \mu_i(T), \hat  \mu_i(T)$ are given in Eq. \eqref{equationmu} and Eq.  \eqref{equationmu1}, respectively.

Vectors $c^0_i(T)$, calculated according to Eq. \eqref{eq_centroid_propagation}, are structured so that, if used as initial conditions for a vectorial max-consensus procedure, the result is the stack vector $c(T)$ containing all the $k$ centroids of step $T$.

\vspace{2mm}

\begin{lemma}
\label{lemma_centroid_propagation}
Let us assume graph $G$ is connected, and let $c(T)$ be the stack vector of all the centroids at step $T$ of the distributed $k$-means algorithm. 
For each time step $T\geq 1$ and for each agent $i=1, \ldots, n$ it holds
$$
c(T)=\mbox{max-consensus}_i(c_i^0(T),c_j^0(T)|\,\, j\neq i,G,n)
$$
where $c_i^0(T)$ is calculated according to Eq. \eqref{eq_centroid_propagation}.
\end{lemma}

\vspace{2mm}

\begin{proof}
At step $T=1$ the only vector $c_i^0(T)$ with components greater than $-\infty$ is  $c^0_{i^*}(T)$ and the result of $n$ iterations of a max-consensus algorithm over a connected graph $G$ is $c(T)=c^0_{i^*}(T)$ for all agents. 
For $T>1$ it can be noted from Eq. \eqref{eq_centroid_propagation}, Eq. \eqref{equationmu} and Eq.  \eqref{equationmu1} that each agent vector $c^0_i(T)$ has finite-valued entries only in the correspondence of the chosen centroid, hence a max-consensus results in a stack vector containing all the centroids. 
\end{proof}

\vspace{2mm}

\begin{remark}Similarly to the centralized $k$-means algorithm, there is the risk that some centroid is not chosen by the agents. In this case the proposed distributed $k$-means algorithm overwrites each of these centroids with a vector whose components are all equal to $-\infty$.

This is due to the particular choice of $c^0_i(T)$; in fact, except for the leader at the first time step, each agent chooses just a centroid and the vector $c^0_i(T)$ is filled with $-\infty$ for the unchosen centers. 
If a center is not selected by any agent, then the corresponding components of $c^0_i(T)$ will be $-\infty$ for each agent.

Such an issue is however easy to fix: if a center $c_j$ is not chosen, then each agent will have a vector  $c^0_i(T)$ such that $c_j=-\infty$. In this case the leader may select another random centroid and a propagation can be implemented via max consensus.
\end{remark}

%\textcolor{red}{Q: This is related to what I mentioned earlier. Is this a known problem for the k means algorithm?  In other words, this issue also surfaces in the centralized implementation, right?}
%
%\textcolor{blue}{A: yes}

\subsubsection{Nearest Centroid Choice and Broadcast} 

Each agent chooses the index $\overline j_{i}$ of the nearest among the current centroids, i.e., 
\begin{equation}
\overline j_{i}=\arg\min_{j=1}^k |c_{ij}(T)-x_i|
\end{equation}
and updates $\mu_i(T)\in \mathbb{R}^d$ and $\hat\mu_i(T)\in \mathbb{R}^d$ as follows:
\begin{equation}
\label{equationmu}
\mu_{ij}(T)=\begin{cases}
1 & \mbox {if } j=\overline j_{i}\\
0 & \mbox {else},
\end{cases}
\end{equation}

\begin{equation}
\label{equationmu1}
\hat \mu_{ij}(T)=\begin{cases}
0 & \mbox {if } j=\overline j_{i}\\
-\infty & \mbox {else}.
\end{cases}
\end{equation}

Then, each agent provides its choice of $\overline j_{i}$ to the agents in its neighborhood $\mathcal{N}_i$.

\subsubsection{Cluster Neighborhood Choice}  
Each agent selects, among its neighbors in $\mathcal{N}_i$, those agents that share the same choice of $\overline j_{i}$. As a result a new neighborhood $\mathcal{N}_i^c(T)\subseteq \mathcal{N}_i$ is obtained.
The new neighborhood $\mathcal{N}_i^c(T)$ is then used to update the centroids.

\subsubsection{Sub-cluster Centroid Identification}  
The new neighborhood induces a subgraph $G^c(T)$ of $G$, where each agent is connected only to agents with the same centroid choice. 

Notice that, as shown in the example of Figure \ref{fig:controesempio}, this does not imply that the agents with the same center choice are connected in  $G^c(T)$ as in general there is no strong relation between the observation associated to each agent and the topology.

\begin{figure}[h]
\begin{center}
\includegraphics[width=3.5in]{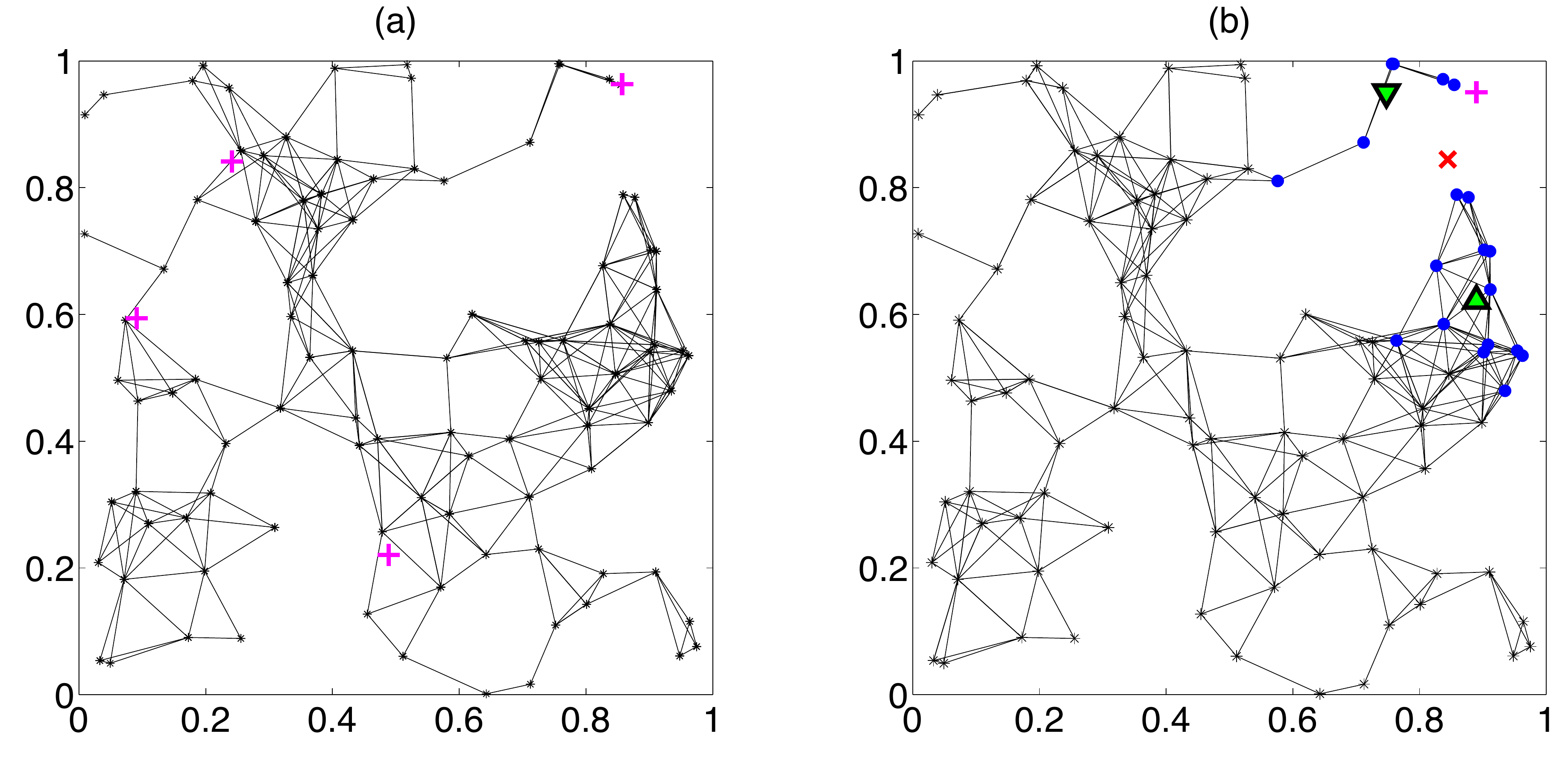}
\caption{Example of incorrect centroid choice over a disconnected subgraph induced by agents with the same centroid choice.}
%Example of disconnected cluster. Figure (a) shows the result of the first iteration of a distributed $k$-means algorithm; Figure (b) shows for the cluster highlighted in blue, which is disconnected in two sub-clusters, the correct center (red cross), the center computed for each sub-cluster (green triangles) and the result of the algorithm (purple plus).}
 \label{fig:controesempio}
\end{center}
\end{figure}

Specifically, the purple crosses in Figure \ref{fig:controesempio}.(a) represent the $k=4$ initial random centroids selected by the leader.
Figure \ref{fig:controesempio}.(b) reports in blue the agents belonging to one of the clusters. The red x represents the correct centroid of the cluster composed of the agents in blue; however note that the blue agents are decomposed in $2$ sub-clusters.

Using the finite time average-consensus algorithm over  $G^c(T)$, therefore, each agent $i$ obtains the centroid $SCC_{\overline j_i \overline h}(T)$ of the sub-cluster $\overline h$ of cluster $\overline j_i$ it belongs to (the green triangles facing upwards and downwards in Figure \ref{fig:controesempio}.(b)).
Notice that iterating the centroid transmission phase would eventually generate an incorrect centroid (i.e., the maximum for each coordinate, represented by the purple cross in Figure \ref{fig:controesempio}.(b)).

Let us now discuss how to cope with such an issue. 

\subsubsection{Sub-cluster Leader Election and Meta Information Gathering}

In order to handle the situation depicted in Figure \ref{fig:controesempio}, the proposed algorithm features a mechanism to calculate the cluster centroids by electing a sub-cluster leader for each sub-cluster and then by calculating the size of each sub-cluster.

The sub-cluster leaders $i^*_{jh}(T)$ are elected for each cluster $j$ and each sub-cluster $h$ by means of a max-consensus algorithm over the graph $G^c(T)$.

The sub-cluster centroids $SCC_{\overline j_i \overline h}(T)$ are obtained by means of a finite-time average-consensus algorithm over $G^c(T)$, where each agent selects its observation $x_i$ as initial condition.

The sub-cluster size $SCS_{jh}(T)$ of each sub-cluster $h$ of a cluster $j$ is eventually calculated by each agent in the sub-cluster by resorting to the approach in \cite{shames2012distributed} over the graph $G^c(T)$, as discussed in Remark \ref{remark:stima_N}; the sub-cluster leader $i^*_{jh}(T)$ of each sub-cluster selects an initial condition equal to one, while other agents in the sub-cluster select zero. 

\vspace{2mm}

\begin{lemma}
\label{lemma_leader_election}
Let us assume graph $G$ is connected, and let $G^c(T)$ be the graph representing the segmentation in communities at step $T$ of the distributed $k$-means algorithm. 
For each time step $T\geq 1$ and for each agent $i=1, \ldots, n$ it holds

$$
SCC_{\overline j_i \overline h}(T)=\mbox{FTA-consensus}_i(x_i,x_j| j\neq i,G^c(T), n),
$$

$$
i^*_{\overline j_i \overline h}=\mbox{max-consensus}_i(i,j| j\neq i,G^c(T), n)
$$

and

\begin{footnotesize}
$$
SCS_{\overline j_i \overline h}(T)=\frac{1}{\mbox{FTA-consensus}_i(SCL_i(T),SCL_j(T)|\,\, j\neq i,G^c(T), n)},
$$
\end{footnotesize}
where $SCC_{\overline j_i \overline h}(T)$, $i^*_{\overline j_i \overline h}$ and $SCS_{\overline j_i \overline h}(T)$ are the centroid, the index of the leader and the size of the sub-cluster $\overline h$ of the cluster $\overline j_i$ agent $i$ belongs to, respectively, and 
$$SCL_i(T)=
\begin{cases}
1 &\mbox{if } i=i^*_{\overline j_i \overline h}\\
0 & \mbox{else}.
\end{cases}
$$
\end{lemma}

\vspace{2mm}

\begin{proof}
Graph $G^c(T)$ contains several disconnected sub-clusters, and since $n$ is an upper bound of the size of each sub-cluster the finite time average-consensus algorithm \cite{sundaram2007finite} provides each agent in a sub-cluster with the component-wise average of the observations belonging to the sub-cluster, i.e., with the sub-cluster centroid $SCC_{\overline j_i \overline h}(T)$.
Similarly using the max-consensus algorithm over $G^c(T)$ for $n$ steps, where the initial state of each agent is its identifier, provides each agent with the identifier $i^*_{\overline j_i \overline h}$ of the leader of the sub-cluster it belongs to.
The finite time average consensus over $G^c(T)$ with initial state equal to $SCL_i(T)$ for all agents, eventually, provides each agent with $1/SCS_{\overline j_i \overline h}(T)$.
\end{proof}

\vspace{2mm}

The centroids are calculated as follows.

\subsubsection{Centroid Calculation}

The agents execute a finite-time average consensus algorithm over $G$ choosing for each agent $i$ an initial condition 
$$\eta_i^0(T)=e_{\overline j_{i}} \otimes \begin{bmatrix}\sigma_i(T)\\ \epsilon_i(T)\end{bmatrix} \in\mathbb{R}^{k(d+1)},$$
where $e_{\overline j_{i}}$ is the $\overline j_{i}$-th vector in the canonical basis and 

$\sigma_i(T)\in\mathbb{R}^d$ is a vector containing the sub-cluster centroid $SCC_i(T)$ weighted by the sub-cluster size $SCS_i(T)$ if agent $i$ is a sub-cluster leader and is zero otherwise, i.e.

$$
\sigma_i(T)=\begin{cases}
SCC_i(T) \cdot SCS_i(T), & \mbox{if } i\, \mbox{is a sub-cluster leader}\\
0_d, & \mbox{else}.
\end{cases}
$$

The scalar $\epsilon_i(T)$ is equal to the sub-cluster size $SCS_i(T)$ if agent $i$ is a sub-cluster leader and is zero otherwise, i.e., 

$$
\epsilon_i(T)=\begin{cases}
SCS_i(T), & \mbox{if } i\, \mbox{is a sub-cluster leader}\\
0, & \mbox{else}.
\end{cases}
$$

\vspace{2mm}

\begin{lemma}
\label{lemma_centroid calculation}
Let us assume graph $G$ is connected. 
For each time step $T\geq 1$ of the distributed $k$-means algorithm and for each agent $i=1, \ldots, n$ it holds

$$
\overline \eta(T) = \mbox{FTA-consensus}_i(\eta_i^0(T),\eta_j^0(T)|\,\, j\neq i,G,n)
$$

where
$$
\overline \eta(T) = \begin{bmatrix}\overline \eta_1(T)\\ \vdots \\ \overline \eta_k(T)\end{bmatrix}
$$
and each term $\overline \eta_j(T)=\begin{bmatrix}\overline \sigma_j(T)^{'}& \overline \epsilon_j(T)\end{bmatrix}^{'}$ is such that 
\begin{equation}
\label{eq_centroid_calculation_lemma}
c_{j}(T)=\frac{\overline \sigma_j(T)}{\overline \epsilon_j(T)},
\end{equation}
where $c_{j}(T)$ is the $j$-th centroid at step $T$.
\end{lemma}

\vspace{2mm}

\begin{proof}
Under the hypothesis that the graph is undirected and connected, the result $\overline \eta(T)$ of the finite time average-consensus for each agent is such that, for all $j=1, \ldots, k$
$$
\overline \eta_j=\begin{bmatrix}\overline \sigma_j(T)\\ \\ \overline \epsilon_j(T)\end{bmatrix}=\frac{1}{n}\begin{bmatrix}\sum_{h=1}^{h^j_{\max}(T)} SCS_{jh}(T)\cdot SCC_{jh}(T)\\ \\
\sum_{h=1}^{h^j_{\max}(T)} SCS_{jh}(T)\end{bmatrix}
$$
where $h^j_{\max}(T)$ is the number of sub-clusters the $j$-th cluster is composed of.
Hence Eq. \eqref{eq_centroid_calculation_lemma} holds true and the proof is complete. 
\end{proof}

\subsubsection{Exit Condition Evaluation}
\label{subsubsec:exitcond}
As discussed at the end of Section \ref{kmeans}, there is the need to define an exit criterion, because the maximum number of iterations $M$ is typically a big quantity and it is preferable to identify an alternative way to terminate the algorithm, in order to save resources.

The stopping criterion $2)$ discussed at the end of Section \ref{kmeans} (stop if there is no variation in the association variables $r_{ij}(T)$ with respect to step $T-1$) can be easily implemented in a distributed fashion. Specifically, each agent $i$ executes a max-consensus algorithm with the following initial condition:

$$
\nu_i^0(T)= \begin{cases}
0 & \mbox{if } \overline j_{i}(T)=\overline j_{i}(T-1)\\
1 & \mbox{else}.
\end{cases}
$$

If the result of the max-consensus is $\overline \nu_i=0$, each agent concludes that there has been no variation in the assignment of the observations to the centroids from step $T-1$ to step $T$, hence the algorithm can be terminated.

The stopping criterion $3)$ is $$|D(T)-D(T-1)|<\Delta_{\max} \Rightarrow stop$$
and can be implemented in a distributed way by means of a finite-time average-consensus where the initial condition of each agent is 
$$\nu_i(0)=n|c_{i\overline j_{i}}(T)-x_i|^2.$$ 

The result of such a procedure is 
$\overline \nu_i=D(T)$, 
hence keeping track of $D(T-1)$ each agent is able to verify if the criterion is satisfied.

Notice that, since the exit conditions may not be met, each agent needs a counter to keep trace of the number of iterations in order to eventually stop when $M$ iterations are executed, where $M$ is the same for all the agents.

\section{Correctness and Time/Memory Complexity}
\label{correctness_computational_complexity}
In this Section we discuss the formal correctness of the distributed $k$-means, i.e., the fact that it responds  to the same specifications as the centralized $k$-means algorithm and that the two algorithms provide the same results if the initial centroids are the same.
We also inspect the time and memory complexity of the distributed $k$-means, and provide a comparison with the centralized one.

\subsection{Correctness}

\begin{theorem}
\label{correctness}
Let us assume that $G$ is connected and undirected, and that each agent knows an upper bound $\tilde n$ for $n$.
Suppose further that the distributed $k$-means algorithm selects the initial centroids $c_{i^*j}(0)$ for $j=1, \ldots, k$. The distributed $k$-means is correct, i.e., it provides the same results as the standard $k$-means algorithm when the same initial centroids $c_{i^*j}(0)$ are chosen.
\end{theorem}

\vspace{2mm}

\begin{proof}
If $G$ is connected and each agent knows an upper bound $\tilde n$ for $n$, then any finite time average-consensus or max-consensus procedure  (possibly with vectorial state for each agent) over $G$ will succeed.

Let us consider a generic step $T$ of the main cycle.
By Lemma 1, the Centroid Propagation phase provides each agent with a vector containing the current $k$ centroids (for the first step the random centroids selected by the leader are obtained).

As a result of the Nearest Centroid Choice, Nearest Centroid Choice Broadcast and Cluster Neighborhood Choice phases, each agent selects a neighborhood $\mathcal{N}_i^c(T)\subseteq \mathcal{N}_i$ which induces a graph $G^c(T)$ which is possibly decomposed in sub-clusters.
This is equivalent to choosing a value for the decision variables $r_{ij}(T)$: since each agent $i$ selects a cluster index $\overline j_i$, in fact, it is easy to note that the constraints of the $k$-means algorithm in Eq. \ref{optptob} are satisfied.

By Lemma 2 and Lemma 3 the centroids are correctly as the component-wise weighted average of the observations associated to the agents in each cluster, exactly as done by the centralized version of the $k$-means algorithm, hence the proof is complete.

\end{proof}

\subsection{Time and Memory Complexity}
\label{time_memory_complexity}

The following result characterizes the time complexity of the distributed $k$-means algorithm.

\begin{proposition}
The time complexity of the distributed $k$-means algorithm is $O(dkn^2M)$.
\end{proposition}
\begin{proof}
In the worst case the main cycle of the distributed $k$-means algorithm is executed $M$ times, and for each iteration a combination of max-consensus and finite time average-consensus algorithms with vectorial initial conditions is done, hence the complexity of the algorithm is $O(\phi n^2M)$, where $\phi$ is the sum of the dimension of the initial conditions of the different consensus algorithms executed during the main cycle.

Specifically, the size of the initial conditions of the consensus algorithms used in each phase is as follows:
\begin{itemize}
\item $dk$ for the centroid propagation phase;
\item $d$ for the centroid subcluster identification phase;
\item $1$ for the subcluster size calculation phase;
\item $(d+1)k$ for the subcluster meta information gathering phase;
\item $1$ for the exit condition evaluation phase.
\end{itemize}
hence it holds $\phi=(2d+1)k+d+2=O(dk)$ and the proof is complete.

\end{proof}

\vspace{2mm}

Notice that the time complexity of the distributed $k$-means algorithm is $O(n)$ times the time complexity of the standard $k$-means algorithm.
This is expected since the distributedness of the proposed setup implies that the transmission of a message from a agent to another may require a number of steps that is equal to the diameter of the graph, hence proportional to the number of agents.

However, the decentralization introduces benefits in terms of robustness (i.e., a central authority is prone to failures), energy saving (i.e., all communications are local) and memory, as emphasized by the following proposition.

\vspace{2mm}

\begin{proposition}
The memory occupancy of the distributed $k$-means algorithm is $O(kd)$, which is \text{$O(k/(n+k))$ times} the memory occupancy of the centralized $k$-means algorithm.
\end{proposition}
\begin{proof}
In the distributed setup, a agent needs to store just the value of its piece of information $x_i\in\mathbb{R}^d$ and $k$ current centroids, each in $\mathbb{R}^d$, therefore the memory complexity in this case is $O(kd)$. In the centralized setting, conversely, the processing unit needs to store also the value of each $x_j$, hence the memory complexity is $O((k+n)d)$.
The ratio of the above memory complexities proves the statement. 
\end{proof}

\vspace{2mm}

\begin{remark}
The distributed $k$-means algorithm has remarkably smaller memory complexity with respect to the centralized setting. In fact, since $n\geq k$ it holds
 
$$0<\frac{k}{n+k}\leq \frac{1}{2}.$$
where values near zero are attained for $n>>k$, while in the worst case there is a reduction of $50\%$ of the memory complexity when moving from a centralized to a distributed setup.

Notice that the above ratio does not depend on $d$, hence the reduction in memory complexity holds both in high-dimensional and low-dimensional contexts.
\end{remark}

\section{Simulation Results}

%\textcolor{blue}{I rewrote this Section based on the changes of the algorithm and on the comments. I also added something about the cycle and the length of each phase of the algorithm. It is still tentative. Let me know your opinion about this.}

\begin{figure*}
\begin{center}
\includegraphics[width=7in]{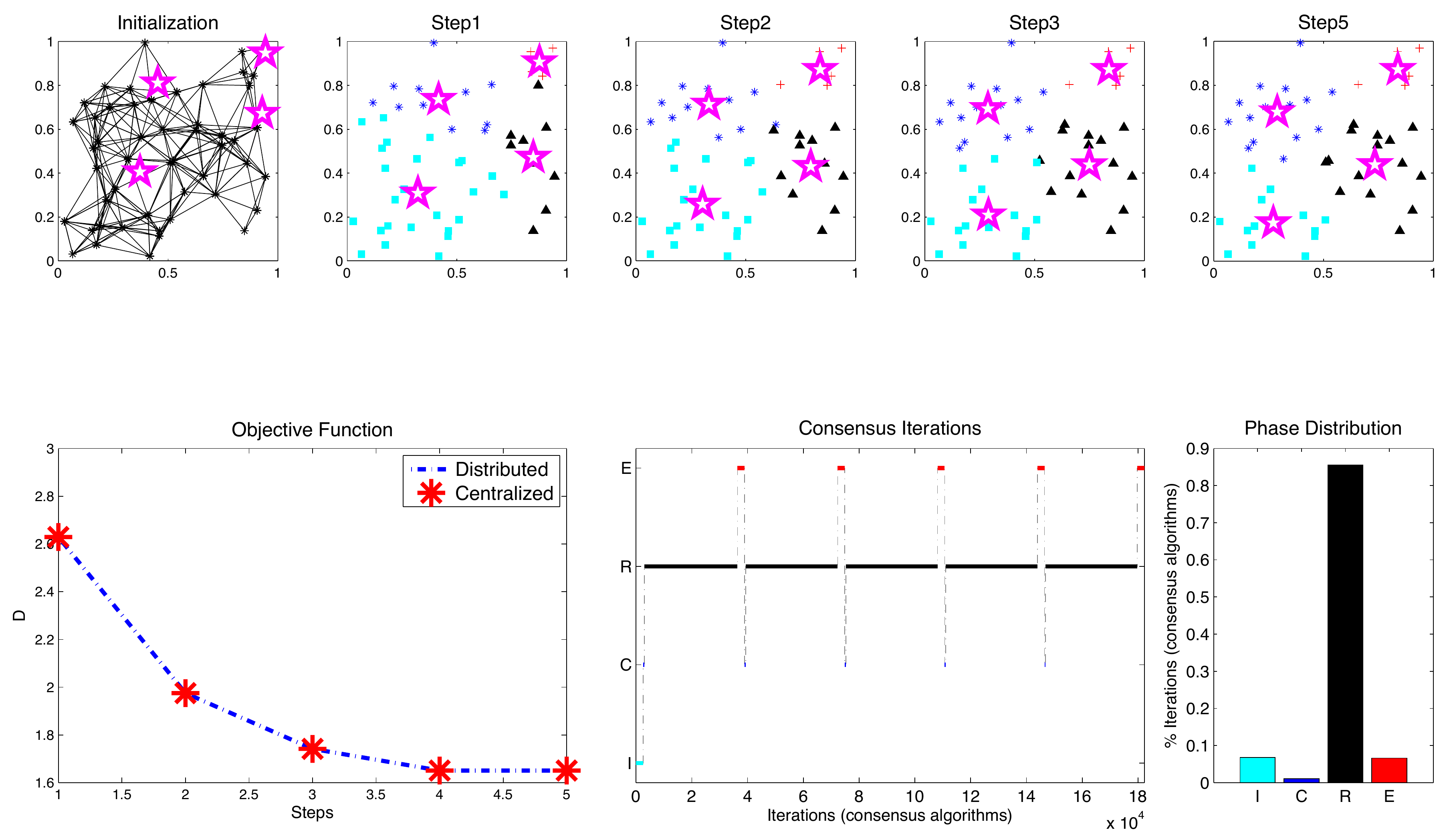}
\caption{Example of execution of the distributed $k$-means algorithm where the state of the agents coincides with their position.}
 \label{fig:positionclustering}
\end{center}
\end{figure*}

\label{results}
\begin{figure*}
\begin{center}
\includegraphics[width=6in]{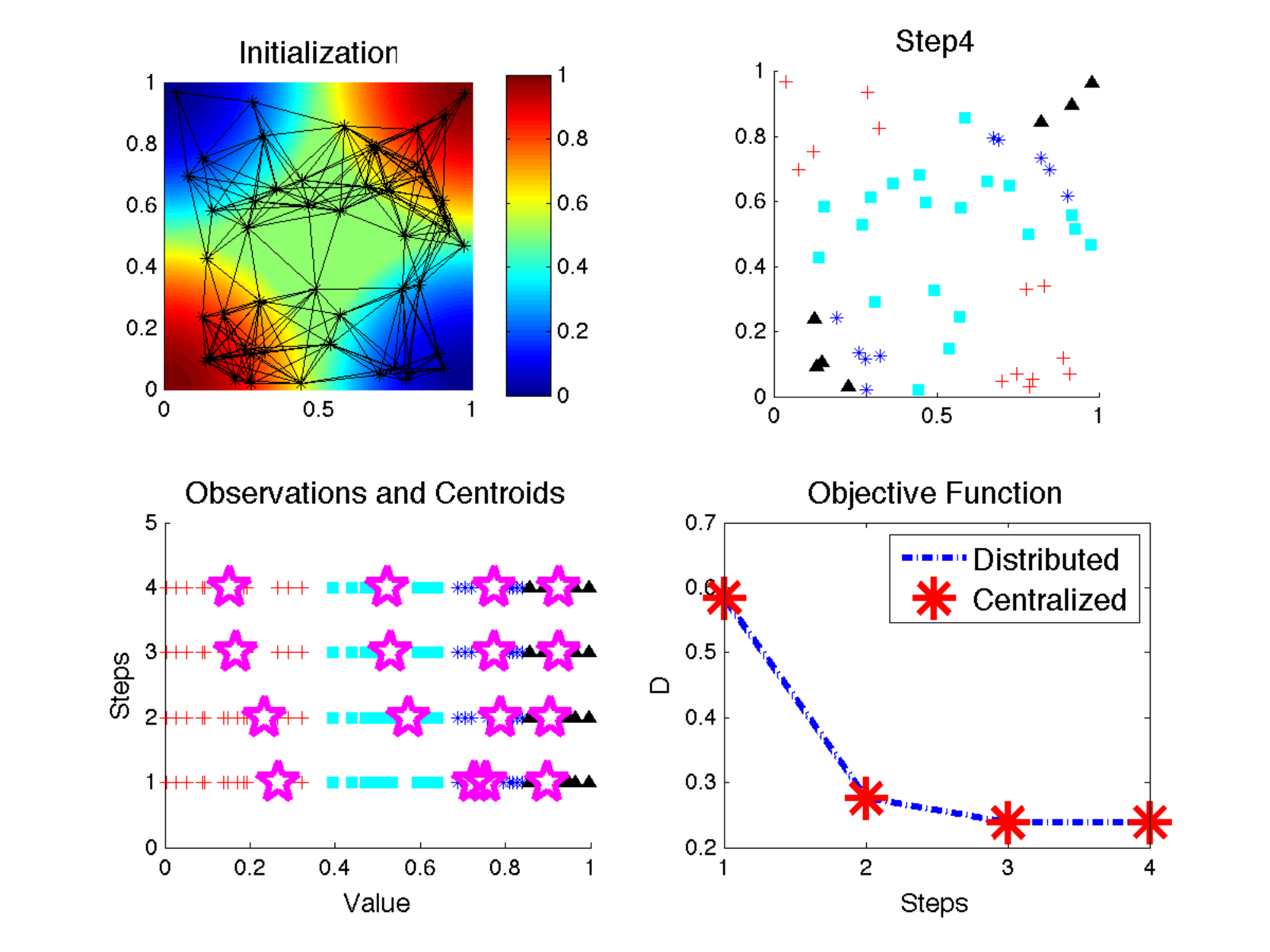}
\caption{Example of execution of the distributed $k$-means algorithm with scalar states. The state of each agent is given by the field reported in the upper left figure.}
 \label{fig:measureclustering}
\end{center}
\end{figure*}

%\begin{figure*}
%\begin{center}
%\includegraphics[width=7in]{./highdim.pdf}
%\caption{Example of execution of the distributed $k$-means algorithm where the state of each agent has high-dimension $d=10$.}
% \label{fig:hybridclustering}
%\end{center}
%\end{figure*}
In this section we provide some examples of application of the proposed distributed $k$-means algorithm.

We consider a situation where $n=50$ agents are embedded in the unit square in $\mathcal{R}^2$, and the agents are connected by means of a unit disk graph topology, that is, each pair of agents with distance less than a threshold $\rho$ is connected; we choose $\rho=\frac{\sqrt 2}{5}$ (i.e., $20\%$ of the maximum possible distance).  We want to partition the agents into $k=4$ clusters. We choose the second stop criterion discussed at the end of Section \ref{kmeans}.

Figure \ref{fig:positionclustering} shows the case where the observations coincide with the agent's positions. 
Specifically the upper plots in the figure show the clustering obtained with respect to the spatial distribution of the agents (the different colors and shapes represent different clusters, while the centers are reported with purple empty stars); in this case it is evident that the agents are clustered according to their position.
The lower left plot shows the evolution of the functional $D$ during the execution of the algorithm. The result of a centralized $k$-means is reported with red asterisks; it is evident that at each step the distributed formulation provides the same results as the centralized one.
The lower central plot shows the iterations of consensus algorithms required to implement the distributed $k$-means algorithm, and for each iteration the corresponding logical phase of the $k$-means algorithm is reported:  initialization (I), centroid choice (C), centroid refinement (R) and exit condition evaluation (E). 
According to the discussion provided in Section \ref{time_memory_complexity}, the number of iterations is $O(dkn^2M)\approx 10^5$ which is $n=50$ times more the time complexity of the centralized $k$-means algorithm. As for the memory complexity, instead, each agent needs to store order of $dk=8$ values, while the centralized algorithm requires order of $(k+n)d=108$, hence around $13.5$ times more information. 
It can be noted that each step of the distributed $k$-means algorithm requires exactly the same number of consensus iterations, because finite time consensus algorithms are used. 

Notice further that the centroid refinement phase is the most time consuming, because of the presence of sub-clusters and because of the combination of max-consensus and finite time average-consensus algorithms. This is also highlighted by the lower rightmost picture, which shows the percentage of consensus iterations used for each of the above four phases; as shown by the picture approximatively $85\%$ of consensus iterations are required to implement the centroid refinement phase.

%
%As for the time complexity,we have $O(dkn^2M)\approx 10^5$ which is $n=50$ times more the time complexity of the centralized $k$-means algorithm. As for the memory complexity, instead, each agent needs to store order of $dk=4$ values, while the centralized algorithm requires order of $(k+n)d=54$, hence again $13.5$ times more information. 

%scalare
Figure \ref{fig:measureclustering} shows an example where the agents have scalar observations in $[0,1]$, which depend on the position of the agents according to the field reported in the upper left figure, 
together with the graph topology. The upper right figure displays the final clustering obtained at step $T=4$ (again, the different colors and shapes represent different clusters). In this case the clustering pattern depends on the field, and it can be observed that very far agents belong to the same cluster.

The lower left plot shows the clustering with respect to the distribution of the observations.
Specifically, the abscissa represents the value of the observations and the ordinate represents the step of the distributed $k$-means algorithm.
Different colors and shapes represent different clusters, while the centroids are reported with purple empty stars.
The plot implies that the algorithm is indeed effective in partitioning the agents according to the observations.

The right lower leftmost plot shows the results in terms of minimization of the objective function D of Eq. \eqref{optptob} obtained as a result of the distributed (blue dotted line) and centralized (red asterisks) $k$-means algorithms. According to Theorem \ref{correctness}, the distributed $k$-means algorithm correctly provides the same results of the centralized $k$-means algorithm in terms of minimization of $D$.

%Let us consider an high dimensional case, where $d=10$ and each component of the observations is random and uniformly distributed in $[0,1]$: in this case, each agent needs to store order of $dk=40$ values, while the centralized algorithm requires order of $(k+n)d=540$, again $13.5$ times more information. In this case, the centroid refinement phase requires a higher percentage of consensus iterations with respect to the previous cases, being around $93\%$.

\section{Conclusions and Future Work}
\label{conclusions}

In this paper we provide a fully decentralized and distributed implementation of the $k$-means clustering algorithm considering a set of agents each of which is equipped with a possibly high-dimensional piece of information or set of measurements.

The proposed algorithm, is proven to be a correct implementation of the $k$-means algorithm in that it provides the same results in terms of centroids, associations and objective function.
We also inspect the time and memory complexity of the proposed distributed implementation showing that the time complexity is $O(n)$ times more the one of the centralized version, where $n$ is the number of agents; the memory complexity, conversely, is considerably reduced from $O((k+n)d)$ to $O(kd)$, where $k$ is the number of clusters and $d$ is the dimension of the observations, hence the proposed algorithm can be implemented using computational units with less resources with respect to the centralized version.

The proposed algorithm finds a direct application in distributed data filtering/classification problems and in mobile robot coordination problems.

Further work will address the analysis of the robustness of the algorithm for noisy measurements and faults at some agents. We will also inspect dynamically changing networks.

Another interesting expected future work direction is the implementation of a distributed procedure to obtain a soft clustering, i.e., a partition where each agent belongs to different clusters with different degrees of membership.
This would contribute to capture more sophisticated aspects of the relations that exist among the agents:  each agent would then belong to each cluster with a different intensity. 

As a result a agent, while belonging to a primary cluster, may eventually contribute to the operations of one or more other clusters, e.g., in the case of an emergency or under specific circumstances. 
The above extension, however, 
appears to be a non-trivial challenge, since the proposed algorithm largely relies on the intrinsic binary association of observations to clusters and during the evolution of the algorithm the topology is segmented in disconnected components, while a soft clustering would require, rather than segmentation, complex weight adjusting schemes.

A final foreseen research direction is to devise more efficient finite time average-consensus algorithms, in order to lower the complexity of the distributed $k$-means algorithm.

\ifCLASSOPTIONcaptionsoff
  \newpage
\fi



%\bibliographystyle{IEEEtran}

\bibliographystyle{IEEEtran}
\bibliography{example}

\end{document}